%% file: CGCN2.tex
\documentclass[11pt,letterpaper]{amsart}

\usepackage{iclr2018_conference,times}
\usepackage{url,amsthm,amsmath,amssymb,amsfonts}
\usepackage{xcolor}
\usepackage{algorithm,algpseudocode}
\usepackage{tikz}
\usepackage{rbase} 
\usepackage{bbm}
\usepackage{paralist}
\definecolor{hcitecolor}{RGB}{40,40,100}

\input{tikzheader}

\algdef{SE}[DOWHILE]{Do}{doWhile}{\algorithmicdo}[1]{\algorithmicwhile\ #1}%
\algnewcommand{\varalg}{\textit}

\title{Covariant Compositional Networks for Learning Graphs}

\author{Risi Kondor, Hy Truong Son, Horace Pan \& Brandon Anderson  \\
Department of Computer Science\\
The University of Chicago\\
Chicago, IL - 60637 \\
\texttt{\{risi,hytruongson,hopan,brandona\}@cs.uchicago.edu} \\
\And
Shubhendu Trivedi \\
Toyota Technological Institute \\
Chicago, IL - 60637\\
\texttt{shubhendu@ttic.edu}
}

\newtheorem{theorem}{Theorem}

\newtheorem{proposition}[theorem]{Proposition}

\newtheorem{definition}{Definition}

\newenvironment{pfprop}[1]{\textbf{Proof of Proposition \ref{#1}.}}{\mbox{}\hfill\m{\blacksquare}\\ \mbox{}\noindent}

\makeatletter
\def\thm@space@setup{%
  \thm@preskip=6pt plus 0pt minus 0pt
  \thm@postskip=0pt plus 0pt minus 0pt 
}
\makeatother

\newcommand{\includepic}[2]{\includegraphics[width=#1\textwidth]{#2}}

\newcommand{\neuron}{\mathfrak{n}}
\newcommand{\labl}{l}
\newcommand{\prt}{\Pcal}

\begin{document}

\maketitle

\begin{abstract}
\input{abstract}

\end{abstract}

\input{introduction}

\input{graphs}

\input{compnets}
\input{covariant2}

\input{operations}
\input{experiments}
\input{conclusion}

\bibliography{CGCN}
\bibliographystyle{iclr2018_conference}

\appendix

\input{app-background}
\input{appendix_proofs}

\end{document}

%% file: tikzheader.tex
\usepackage{tikz}
\usetikzlibrary{shadows,matrix, calc, positioning, arrows,shapes,backgrounds, arrows.meta, quotes}
\usepackage{gensymb,stmaryrd,mathtools,textcomp,xspace,grffile,rbase}
\usetikzlibrary{shapes.geometric,fit,matrix,positioning,shapes.multipart}
\usetikzlibrary{decorations.markings}
\usetikzlibrary{decorations.pathreplacing,fadings}
\pgfdeclarelayer{background}
\pgfsetlayers{background,main}
\usetikzlibrary{topaths, calc,3d}
\usetikzlibrary{fit}

\tikzset{cnode/.style={circle,draw=magenta!75!red,fill=white,thick,minimum size=10pt,circular drop shadow}}

\tikzset{cnodesmall/.style={circle,draw,thick,minimum size=.25em}}

\tikzset{middlearrow/.style={
		decoration={markings,
			mark= at position 0.3 with {\arrow{#1}} ,
		},
		postaction={decorate}
	}
}

\tikzset{
	annotated cuboid/.pic={
		\tikzset{%
			every edge quotes/.append style={midway, auto},
			/cuboid/.cd,
			#1
		}
		\draw [every edge/.append style={pic actions, densely dashed, opacity=0}, pic actions]
		(0,0,0) coordinate (o) -- ++(-\cubescale*\cubex,0,0) coordinate (a) -- ++(0,-\cubescale*\cubey,0) coordinate (b) edge coordinate [pos=1] (g) ++(0,0,-\cubescale*\cubez)  -- ++(\cubescale*\cubex,0,0) coordinate (c) -- cycle
		(o) -- ++(0,0,-\cubescale*\cubez) coordinate (d) -- ++(0,-\cubescale*\cubey,0) coordinate (e) edge (g) -- (c) -- cycle
		(o) -- (a) -- ++(0,0,-\cubescale*\cubez) coordinate (f) edge (g) -- (d) -- cycle;
		\path [every edge/.append style={pic actions, |-|}]
		;
	},
	/cuboid/.search also={/tikz},
	/cuboid/.cd,
	width/.store in=\cubex,
	height/.store in=\cubey,
	depth/.store in=\cubez,
	units/.store in=\cubeunits,
	scale/.store in=\cubescale,
	width=10,
	height=10,
	depth=10,
	units=cm,
	scale=.1,
}

\colorlet{mewnol}{red!75!magenta}%
\colorlet{allanol}{blue!50!black}%

%% file: abstract.tex
Most existing neural networks for learning graphs address permutation invariance by conceiving 
of the network as a message passing scheme, where each node sums the feature vectors coming from its neighbors. 
We argue that this imposes a limitation on their representation power, and instead propose a new 
general architecture for representing objects consisting of a hierarchy of parts, which we 
call \emph{covariant compositional networks} (CCNs). Here, covariance means that the activation of each neuron must 
transform in a specific way under permutations, similarly to steerability in CNNs. 
We achieve covariance by making each activation transform according to a tensor representation of the 
permutation group, and derive the corresponding tensor aggregation rules that each neuron must implement. 
Experiments show that CCNs can outperform competing methods on standard graph learning benchmarks.

%% file: introduction.tex
\section{Introduction}

Learning on graphs has a long history in the kernels literature, including approaches based on 
random walks \citep{Gaertner02, BorKri05, FerEtal13}, counting subgraphs \citep{ShervEtal09}, spectral 
ideas  \citep{VishyRisi2010}, label propagation schemes with hashing \citep{Shervashidze2011,NeumannEtal},  
and even algebraic ideas \citep{KondorBorgwardt2008}. 
Many of these papers address moderate size problems in chemo- and bioinformatics, 
and the way they represent graphs is essentially fixed. 

Recently, with the advent of deep learning and much larger datasets, 
a sequence of neural network based approaches have appeared to address the same problem, 
starting with \citep{Scarselli}.  
In contrast to the kernels framework, neural networks 
effectively integrate the classification or regression problem at 
hand with learning the graph representation itself, in a single, end-to-end system. 
In the last few years, there has been a veritable  explosion in research activity in this area. Some
of the proposed graph learning architectures 
\citep{Duvenaud2015, Kearns2016, Niepert2016} directly seek inspiration from the type of classical  
CNNs that are used for image recognition \citep{LeCun1998, Hinton2012}. 
These methods involve first fixing a vertex
ordering, then moving a filter across vertices while doing some computation as a function of the
local neighborhood to generate a representation. This process is then repeated multiple times like in
classical CNNs to build a deep graph representation. Other notable works on graph neural
networks include \citep{LiZemel2015, Schutt2017, Battaglia2016, KipfWelling2017}. 
Very recently, \citep{Gilmer2017} showed that many of these approaches can be seen to be specific
instances of a general message passing formalism, and coined the term 
\emph{message passing neural networks} (MPNNs) to refer to them collectively. 
\ignore{
Viewed through the lens of this common framework, these approaches can be seen to
differ in specification of the message passing, update or readout phases. It is also to be noted
that some authors also presents MPNNs as generalizations of convolutional networks due to their
similarity with image based CNNs, which operate on grid graphs. Some features that seem to indicate
this similarity are: a) multiple feature maps generated by some local computation and b) the
effective area of influence of each neuron increases as we go deeper in the network.
}

While MPNNs have been very successful in applications 
and are an active field of research, they differ from classical CNNs in a fundamental way:
the internal feature representations in CNNs are \emph{equivariant} to such transformations of 
the inputs as translation and rotations \citep{CohenWelling2016, CohenWelling2017}, 
the internal representations in MPNNs are fully invariant. 
This is a direct result of the fact that MPNNs deal with the permutation invariance 
issue in graphs simply by summing the messages coming from each neighbor. 
In this paper we argue that this is a serious limitation that restricts the representation power of MPNNs. 

\ignore{
To intuitively understand the nature of this limitation, it is useful to consider
the activation $f_{\ell}(x)$ for neuron $x$ in a higher layer of a CNN. The activation of the neuron
transforms in a controlled manner when there is (for example) a translation of the input image. In
the case of graphs, it is essential that the activations transform in a controlled linear manner
when the labels of the graph are permuted. In MPNNs, however, on permuting the labels, there is no
transformation of the activation at all. To put it a bit differently, if classical CNNs followed a
recipe similar to MPNNs for transformation that matter for improving sample complexity in the image
recognition domain, then in higher layers the filter $\chi_{\ell}$ would be symmetric and invariant
to (say) 90 degree rotations, which is an absurd constraint.
}

MPNNs are ultimately compositional (part-based) models, that build up the representation of the graph from the 
representations of a hierarchy of subgraphs. To address the covariance issue, we study 
the covariance behavior of such networks in general, introducing a new general class of neural network 
architectures, which we call \emph{compositional networks} (comp-nets). One advantage of this generalization is that 
instead of focusing attention on the mechanics of how information propagates from node to node, 
it emphasizes the connection to convolutional networks, in particular, it shows that 
what is missing from MPNNs is essentially the analog of \emph{steerability}. 

Steerability implies that the activations (feature vectors) at a given neuron must transform according to 
a specific representation (in the algebraic sense) of the symmetry group of its receptive field, in 
our case, the group of permutations, \m{\Sbb_m}. In this paper we only consider the defining representation 
and its tensor products, leading to first, second, third etc.\:order \emph{tensor activations}. 
We derive the general form of covariant tensor propagation in comp-nets, and find that 
each ``channel'' in the network corresponds to a specific way of contracting a higher order tensor to a lower 
order one. Note that here by \emph{tensor activations} we mean not just that each activation is expressed as 
a multidimensional array of numbers (as the word is usually used in the neural networks literature), but also 
that it transforms in a specific way under permutations, which is a more stringent criterion. 
The parameters of our covariant comp-nets are the entries of the mixing matrix that prescribe how these channels 
communicate with each other at each node. 
Our experiments show that this new architecture can beat scalar message passing neural networks on 
several standard datasets.

%% file: graphs.tex
\section{Learning graphs}

Graph learning encompasses a broad range of problems where the inputs are graphs and the outputs 
are class labels (classification), real valued quantities (regression) or more general, possibly 
combinatorial, objects. 
In the standard supervised learning setting this means that the training set consists of 
\m{m} input/output pairs \m{\cbrN{(G_1,y_1),(G_2,y_2),\ldots,(G_m,y_m)}},  
where each \m{G_i} is a graph and \m{y_i} is the corresponding label, 
and the goal is to learn a function \m{h\colon G\to y} that will successfully predict 
the labels of further graphs that were not in the training set.

\ignore{
The most classical setting for this type of problem is chemoinformatics, where 
the objective is to predict certain properties of molecules (toxic/non-toxic, binding/non-binding 
or physical properties such as boiling point, specific heat, etc) from their 
chemical structure. However, recently, a host of other innovative applications have 
appeared in the literature from .... .
At a more general level, one can argue that graph learning has always been a paradigmatic problem 
on the combinatorial side of machine learning, because graphs are such a general type of combinatorial 
structure, yet they are so easy to conceptualize. Many other problems have natural reductions to  
graph learning, for example... 
}

By way of fixing our notation, in the following we assume the each graph \m{G} is a pair 
\m{(V,E)}, where \m{V} is the vertex set of \m{G} and \m{E\subseteq V\times V} is its edge set. 
For simplicity, we assume that \m{V=\cbrN{\oneton{n}}}. We also assume that \m{G} has no 
self-loops (\m{(i,i)\not\in E} for any \m{i\tin V}) and that \m{G} is symmetric, 
i.e., \m{(i,j)\tin E \Rightarrow (j,i)\tin E}\footnote{Our framework has natural generalizations 
to non-symmetric graphs and graphs with self-loops, but in the interest of keeping our discussion 
as simple as possible, we will not discuss these cases in the present paper.}. 
We will, however, allow each edge \m{(i,j)} to have a corresponding weight 
\m{w_{i,j}}, and each vertex \m{i} to have a corresponding feature vector (vertex label) \m{\labl_i\tin \RR^d}. 
The latter, in particular, is important in many scientific applications, where \m{\labl_i} might encode, 
for example, what type of atom occupies a particular site in a molecule, or the identity 
of a protein in a biochemical interaction network. 
All the topological information about \m{G} can be summarized in an adjacency 
matrix \m{A\tin\RR^{n\times n}}, where \m{A_{i,j}=w_{i,j}} if \m{i} and \m{j} are connected by 
an edge, and otherwise \m{A_{i,j}\<=0}. When dealing with labeled graphs, 
we also have to provide \m{(\sseq{\labl}{n})} to fully specify \m{G}. 

One of the most fascinating aspects of graphs, but also what makes graph learning 
challenging, is that they involve structure at multiple different scales. 
In the case when \m{G} is the graph of a protein, for example, an ideal graph learning algorithm 
would represent \m{G} in a manner that simultaneously captures structure at the level of individual atoms, 
functional groups, interactions between functional groups, subunits of the protein, and the 
protein's overall shape. 

The other major requirement for graph learning algorithms relates to the fact that the usual ways to store 
and present graphs to learning algorithms have a critical spurious symmetry: 
If we were to permute the vertices of \m{G} by any permutation \m{\sigma\colon\cbrN{\oneton{n}}\to\cbrN{\oneton{n}}} 
(in other words, rename vertex \m{1} as \m{\sigma(1)}, vertex \m{2} as \m{\sigma(2)}, etc.), then 
the adjacency matrix would change to 
\[A'_{i,j}=A_{\sigma^{-1}\nts(i),\sigma^{-1}\nts(j)},\]
and simultaneously the vertex labels would change to \m{(\sseq{\labl'}{n})}, 
where \m{{\labl'}_i\<={\labl}_{\sigma^{-1}\nts(i)}}. 
However, \m{G'\<=(A',\sseq{\labl'}{n})} would still represent exactly the same graph as 
\m{G\<=(A,\sseq{\labl}{n})}.  
In particular, (a) in training, whether \m{G} or \m{G'} is presented to the algorithm 
must not make a difference to the final hypothesis \m{h} that it returns,  
(b) \m{h} itself must satisfy \m{h(G)=h(G')} for any labeled graph and its permuted variant.  

Most learning algorithms for combinatorial objects hinge on some sort of fixed or learned internal 
representation of data, called the \emph{feature map}, which, in our case we denote \m{\phi(G)}. 
The set of all \m{n!} possible permutations of \m{\cbrN{\oneton{n}}} forms a group called the 
symmetric group of order \m{n}, denoted \m{\Sn}.   
The permutation invariance criterion can then be formulated as follows  (Figure \ref{fig: invariance}). 
\input{fig-invariance}

\begin{definition}\label{def: invariant}
Let \m{\Acal} be a graph learning algorithm that uses a feature map \m{G\mapsto \phi(G)}. 
We say that the feature map \m{\phi} (and consequently the algorithm \m{\Acal}) is \df{permutation invariant} 
if, given any \m{n\tin\NN}, any \m{n} vertex labeled graph \m{G=(A,\sseq{\labl}{n})}, 
and any permutation \m{\sigma\tin\Sn}, 
letting \m{G'\<=(A',\sseq{\labl'}{n})}, where \m{A'_{i,j}\<=A_{\sigma^{-1}\nts(i),\sigma^{-1}\nts(j)}} 
\,and\: \m{\labl'_i\<=\labl_{\sigma^{-1}\nts(i)}},\, we have that\, \m{\phi(G)\<=\phi(G')}. 
\end{definition}

Capturing multiscale structure and respecting permutation invariance are the two the key constraints 
around which most of the graph learning literature revolves. 
In kernel based learning, for example, invariant kernels have been constructed by counting random walks 
\citep{Gaertner02}, matching eigenvalues of the graph Laplacian \citep{VishyRisi2010} and using 
algebraic ideas \citep{KondorBorgwardt2008}.

%% file: fig-invariance.tex
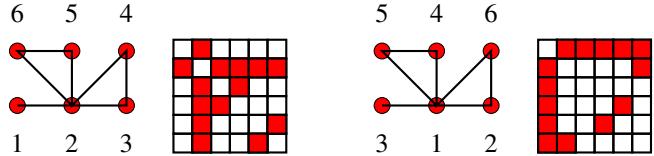
\begin{figure}[t]
\begin{center}

\begin{tikzpicture}
[
box/.style={rectangle,draw=black,thick, minimum size=0.25cm},
]
\draw[fill=red] (0.8,0) circle (3pt);
\draw[fill=red] (1.525,0) circle (3pt);
\draw[fill=red] (2.25,0) circle (3pt);
\draw[fill=red] (0.8,0.725) circle (3pt);
\draw[fill=red] (1.525,0.725) circle (3pt);
\draw[fill=red] (2.25,0.725) circle (3pt);
\node at (0.8,-0.5) {1};
\node at (1.525,-0.5) {2};
\node at (2.25,-0.5) {3};
\node at (2.25,1.225) {4};
\node at (1.525,1.225) {5};
\node at (0.8,1.225) {6};
\draw[thick] (0.8,0) -- (1.525,0) -- (2.25,0) -- (2.25,0.725) -- (1.525,0) -- (1.525,0.725) -- (0.8,0.725) -- (1.525,0);

\foreach \x in {3.0,3.25,3.5,3.75,4.0,4.25}{
	\foreach \y in {-0.5,-0.25,0,0.25,0.5,0.75}
	\node[box] at (\x,\y){};
}

\node[box,fill=red] at (3.25,0.75){};
\node[box,fill=red] at (3.0,0.5){}; \node[box,fill=red] at (3.5,0.5){}; \node[box,fill=red] at (3.75,0.5){}; \node[box,fill=red] at (4.0,0.5){}; \node[box,fill=red] at (4.25,0.5){};
\node[box,fill=red] at (3.25,0.25){}; \node[box,fill=red] at (3.75,0.25){};
\node[box,fill=red] at (3.25,0){}; \node[box,fill=red] at (3.5,0){};
\node[box,fill=red] at (4.25,-0.25){};\node[box,fill=red] at (3.25,-0.25){};
\node[box,fill=red] at (3.25,-0.5){};\node[box,fill=red] at (4.0,-0.5){};

\end{tikzpicture}
\hspace{25pt}
\begin{tikzpicture}
[
box/.style={rectangle,draw=black,thick, minimum size=0.25cm},
]
\draw[fill=red] (0.8,0) circle (3pt);
\draw[fill=red] (1.525,0) circle (3pt);
\draw[fill=red] (2.25,0) circle (3pt);
\draw[fill=red] (0.8,0.725) circle (3pt);
\draw[fill=red] (1.525,0.725) circle (3pt);
\draw[fill=red] (2.25,0.725) circle (3pt);
\node at (0.8,-0.5) {3};
\node at (1.525,-0.5) {1};
\node at (2.25,-0.5) {2};
\node at (2.25,1.225) {6};
\node at (1.525,1.225) {4};
\node at (0.8,1.225) {5};
\draw[thick] (0.8,0) -- (1.525,0) -- (2.25,0) -- (2.25,0.725) -- (1.525,0) -- (1.525,0.725) -- (0.8,0.725) -- (1.525,0);

\foreach \x in {3.0,3.25,3.5,3.75,4.0,4.25}{
	\foreach \y in {-0.5,-0.25,0.0,0.25,0.5,0.75}
	\node[box] at (\x,\y){};
}

\node[box,fill=red] at (3.25,0.75){}; \node[box,fill=red] at (3.5,0.75){}; \node[box,fill=red] at (3.75,0.75){}; \node[box,fill=red] at (4.0,0.75){}; \node[box,fill=red] at (4.25,0.75){};
\node[box,fill=red] at (3.0,0.5){}; \node[box,fill=red] at (4.25,0.5){};
\node[box,fill=red] at (3.0,0.25){};
\node[box,fill=red] at (3.0,0){}; \node[box,fill=red] at (4.0,0){};
\node[box,fill=red] at (3.75,-0.25){}; \node[box,fill=red] at (3.0,-0.25){};
\node[box,fill=red] at (3.25,-0.5){}; \node[box,fill=red] at (3.0,-0.5){};

\end{tikzpicture}
\end{center}
\caption{\label{fig: invariance}
(a) A small graph \m{G} with 6 vertices and its adjacency matrix. 
(b) An alternative form \m{G'} of the same graph, derived from \m{G} by renumbering the 
vertices by a permutation \m{\sigma\colon\cbrN{1,2,\ldots,6}\mapsto \cbrN{1,2,\ldots,6}}.  
The adjacency matrices of \m{G} and \m{G'} are different, but topologically they represent the 
same graph. Therefore, we expect the feature map \m{\phi} 
to satisfy \m{\phi(G)=\phi(G')}.
}
\end{figure}

%% file: compnets.tex
\section{Compositional networks}

Many recent graph learning papers,  
whether or not they make this explicit, employ a \emph{compositional} approach to modeling graphs, 
building up the representation of \m{G} from representations of subgraphs.  
At a conceptual level, this is similar 
to part-based modeling, which has a long history in machine learning 
\citep{Fischler1973,Ohta1978,Tu2005,Felzenszwalb2005,Zhu2006,Felzenszwalb2010}. 
In this section we introduce a general, abstract architecture called \df{compositional networks (comp-nets)}
for representing complex objects as a combination of their parts, and show that several exisiting 
graph neural networks can be seen as special cases of this framework. 

\begin{definition}
Let \m{\Gcal} be a compound object with \m{n} elementary parts (atoms) \m{\Ecal=\cbrN{\sseq{e}{n}}}. 
A \df{composition scheme} for \m{\Gcal} is a directed acyclic graph (DAG) \m{\Mcal} 
in which each node \m{\neuron_i} is associated with some subset \m{\prt_i} of \m{\Ecal}  
(these subsets are called the \df{parts} of \m{\Gcal}) in such a way that 
\begin{compactenum}[~1.]
\item If \m{\neuron_i} is a leaf node, then \m{\prt_i} contains a single atom \m{e_{\xi(i)}}\footnote{
Here \m{\xi} is just a function that establishes the mapping between each leaf node and the corresponding atom.}. 
\item \m{\Mcal} has a unique root node \m{\neuron_r}, which corresponds to the entire set \m{\cbrN{\sseq{e}{n}}}. 
\item For any two nodes \m{\neuron_i} and \m{\neuron_j}, 
if \m{\neuron_i} is a descendant of \m{\neuron_j}, then \m{\prt_i\subset \prt_j}. 
\end{compactenum}
\end{definition}

We define a compositional network as a composition scheme   
in which each node \m{\neuron_i} also carries a \emph{feature vector} \m{f_i}  
that provides a representation of the corresponding part (Figure \ref{fig: parts}). 
When we want to emphasize the connection to more classical neural architectures, we will 
refer to \m{\neuron_i} as the \m{i}'th \df{neuron}, 
\m{\prt_i} as its \df{receptive field}\footnote{ Here and in the following 
by the ``receptive field'' of a neuron \m{\neuron_i} in a feed-forward network we mean the set of all 
input neurons from which information can propagate to \m{\neuron_i}.}, 
and \m{f_i} as its \df{activation}. 
\input{fig-parts}

\begin{definition}\label{def: part based repr}
Let \m{\Gcal} be a compound object in which each atom \m{e_i} carries a label \m{\labl_i},  
and \m{\Mcal} a composition scheme for \m{\Gcal}. 
The corresponding \df{compositional network} \m{\Ncal} is a DAG with the same structure as \m{\Mcal} 
in which each node \m{\neuron_i} also has an associated feature vector \m{f_i} such that 
\begin{compactenum}[~1.]
\item If \m{\neuron_i} is a leaf node, then \m{f_i=\labl_{\xi(i)}}.
\item If \m{\neuron_i} is a non-leaf node, and its children are \m{\neuron_{c_1},\ldots, \neuron_{c_k}}, then 
\m{f_i=\Phi(f_{c_1},f_{c_2},\ldots, f_{c_k})} for some \df{aggregation function} \m{\Phi}. 
(Note: in general, \m{\Phi} can also depend on the relationships between the subparts, 
but for now, to keep the discussion as simple as possible, we ignore this possibility.)  
\end{compactenum}  
The representation \m{\phi(\Gcal)} afforded by the comp-net 
is given by the feature vector \m{f_{r}} of the root. 
\end{definition}

Note that while, for the sake of concreteness, we call the \m{f_i}'s  ``feature vectors'', there is 
no reason a priori why they need to be vectors rather than some other type of mathematical object. 
In fact, in the second half of the paper we make a point of treating the \m{f_i}'s  
as tensors, because that is what will make it the easiest to describe 
the specific way that they transform with respect to permutations. 
\input{fig-invariant-parts} 

In compositional networks for graphs, the atoms will usually be the vertices,  
and the \m{\prt_i} parts will correspond to clusters of nodes or neighborhoods of given radii. 
Comp-nets are particularly attractive in this domain because  
they can combine information from the graph at different scales. 
The comp-net formalism also suggests a natural way to satisfy the permutation invariance 
criterion of Definition \ref{def: invariant}. 

\begin{definition}\label{def: invariant parts}
Let \m{\Mcal} be the composition scheme of an object \m{\Gcal} with \m{n} atoms 
and \m{\Mcal'} the composition scheme of 
another object that is equivalent in structure to 
\m{\Gcal}, except that its atoms have been permuted by some permutation \m{\sigma\tin\Sn} 
(\m{e'_i\<=e_{\sigma^{-1}\nts(i)}} and \m{\ell'_i\<=\ell_{\sigma^{-1}\nts(i)}}). 
We say that \m{\Mcal} (more precisely, the algorithm generating \m{\Mcal}) 
is \df{permutation invariant} 
if there is a bijection \m{\psi\colon \Mcal\to\Mcal'} taking each \m{\neuron_a\tin \Mcal} 
to some \m{\neuron'_b\tin\Mcal'} such that if \m{\prt_a\<=\cbrN{e_{i_1},\ldots,e_{i_k}}}, then 
\m{\prt'_b\<=\cbrN{e'_{\sigma(i_1)},\ldots,e'_{\sigma(i_k)} }}. 
\end{definition}

\begin{proposition}\label{prop: invariant reps} 
Let \m{\phi(\Gcal)} be the output of a comp-net based on a composition scheme \m{\Mcal}. 
Assume  
\begin{compactenum}[~1.]
\item \m{\Mcal} is permutation invariant in the sense of Definition \ref{def: invariant parts}.  
\item The aggregation function \m{\Phi(f_{c_1},f_{c_2},\ldots, f_{c_k})} used to compute the feature 
vector of each node from the feature vectors of its children is 
invariant to the permutations of its arguments. 
\end{compactenum}
Then the overall representation \m{\phi(\Gcal)} is invariant to permutations of the atoms. 
In particular, if \m{\Gcal} is a graph  
and the atoms are its vertices, then \m{\phi} is a permutation invariant graph representation.  
\end{proposition}

\subsection{Message passing neural networks as a special case of comp-nets}\label{sec: MPNN}

Graph learning is not the only domain where invariance and multiscale structure are important: 
the most commonly cited reasons for the success of convolutional 
neural networks (CNNs) in image tasks is their ability to address exactly these two criteria 
in the vision context.  
Furthermore, 
each neuron \m{\neuron_i} in a CNN aggregates information from a small set of neurons 
from the previous layer, therefore its receptive field, 
corresponding to \m{\prt_i}, 
is the union of the receptive fields of its ``children'', 
so we have a hierarchical structure very similar to that described in the previous section. 
In this sense, CNNs are a specific kind of compositional network, where the atoms are pixels. 
This connection has inspired several authors to frame graph learning as a generalization of 
convolutional nets to the graph domain  
\citep{BrunaZaremba2014,HenaffLeCun2015,Duvenaud2015, Defferrard2016, KipfWelling2017}. 
While in mathematics convolution has a fairly specific meaning that is side-stepped by this analogy,  
the CNN analogy does suggest that a natural way to define the \m{\Phi} aggregation functions 
is to let \m{\Phi(f_{c_1},f_{c_2},\ldots, f_{c_k})} be a linear function of 
\m{f_{c_1},f_{c_2},\ldots, f_{c_k}} followed by a pointwise nonlinearity, such as a ReLU operation. 

To define a comp-net for graphs we also need to specify the composition scheme \m{\Mcal}.  
Many algorithms define \m{\Mcal} in layers, where each layer 
(except the last) has one node 
for each vertex \mbox{of \m{G}:} 
\begin{compactenum}[~\Mcal 1.]
\item In layer \m{\ell\<=0} each node \m{\neuron^0_i} represents the single vertex \m{\prt^0_i=\cbrN{i}}. 
\item In layers \m{\ell=1,2,\ldots,L}, node \m{\neuron^\ell_i} is connected to all nodes from the 
previous level that are neighbors of \m{i} in \m{G}, i.e., the children of \m{\neuron^\ell_i} are 
\[\textrm{ch}(\neuron^\ell_i)=\setofN{\neuron^{\ell-1}_j}{j\tin \Ncal(i)},\]
where \m{\Ncal(i)} denotes the set of neighbors of \m{i} in \m{G}. 
Therefore, \m{\prt^\ell_i=\bigcup_{j\in\Ncal(i)}\prt^{\ell-1}_j}. 
\item In layer \m{L\<+1} we have a single node \m{\neuron_r} that represents the entire graph and 
collects information from all nodes at level \m{L}.
\end{compactenum}
Since this construction only depends on topological information about \m{G}, the resulting 
composition scheme is guaranteed to be permutation invariant in the sense of Definition \ref{def: invariant parts}. 

A further important consequence of this way of defining \m{\Mcal} 
is that the resulting comp-net can be equivalently interpreted as \emph{label propagation algorithm}, 
where in each round \m{\ell\<=1,2,\ldots,L}, each vertex aggregates information from its neighbors 
and then updates its own label.  

\begin{algorithm}[H]
\begin{algorithmic} 
\State \texttt{for each vertex }\m{i} 
\State \phantom{MM}\m{f^0_i\leftarrow \labl_i}
\State \texttt{for}\, \m{\ell\<=1}\, \texttt{to}\, \m{L}  
\State \phantom{MM}\texttt{for each vertex }\m{i} 
\State \phantom{MMMM}\m{f^\ell_i\leftarrow \Phi(f^{\ell-1}_{i_1},\ldots,f^{\ell-1}_{i_k})} 
\texttt{ where } \m{\Ncal(i)=\cbrN{\sseq{i}{k}}}
\State \m{\phi(G)\equiv f_r\leftarrow \Phi(f^L_1,\ldots,f^L_n)} 
\end{algorithmic}
\caption{\label{alg: propagation}
The label propagation algorithm corresponding to \Mcal 1--\Mcal 3}
\end{algorithm}

Many authors choose to describe graph neural networks exclusively in terms of label propagation, 
without mentioning the compositional aspect of the model. 
\citet{Gilmer2017} call this general approach \emph{message passing neural networks}, and point out 
that a range of different graph learning architectures are special cases of it. 
More broadly, the classic Weisfeiler--Lehman test of isomorphism also follows the same logic 
\citep{WL1968,Read1977, CaiImmerman1992}, 
and so does the related Weisfeiler--Lehman kernel, arguably the most successful 
kernel-based approach to graph learning \citep{Shervashidze2011}. 
Note also that in label propagation or message passing algorithms 
there is a clear notion of the \emph{source domain} 
of vertex \m{i} at round \m{\ell}, as the set of vertices that can influence \m{f^\ell_i}, and 
this corresponds exactly  
to the receptive field \m{\prt^\ell_i} of ``neuron'' \m{\neuron^\ell_i} in the comp-net picture. 

The following proposition is immediate from the form of Algorithm \ref{alg: propagation} and reassures 
us that message passing neural networks, as special cases of comp-nets, 
do indeed produce permutation invariant representations of graphs. 

\begin{proposition}
Any label propagation scheme in which the aggregation function \m{\Phi} is invariant to the permutations 
of its arguments is invariant to permutations in the sense of Definition \ref{def: invariant}. 
\end{proposition} 

In the next section we argue that invariant message passing networks are limited in their representation 
power, however, and describe a generalization via comp-nets that overcomes some of these limitations. 

%% file: fig-parts.tex
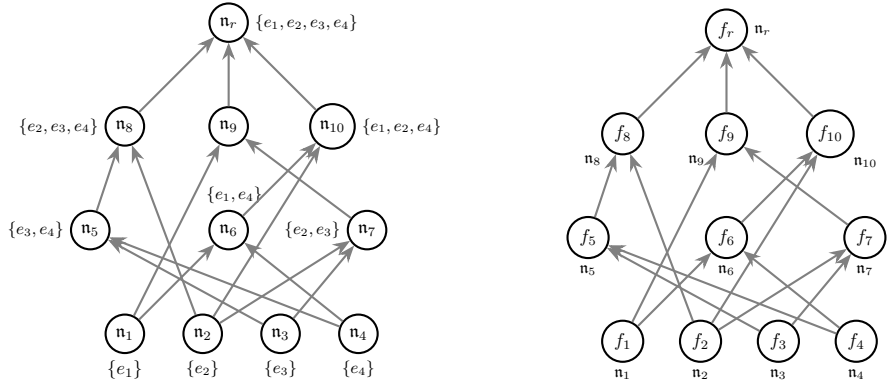
\begin{figure}[t!]
\begin{center}

\begin{tikzpicture}[scale = 0.23, every node/.style={scale=0.7},level/.style={},decoration={brace,mirror,amplitude=7}, >={Stealth[gray]}]

 
\node (0dot1) [cnodesmall,draw=black,thick,label={below: $\{e_1\}$}] at (-6,-4) {$\textcolor{black}{\mathfrak{n}_{1}}$};
\node (0dot2) [cnodesmall,draw=black,thick,label={below: \small $\{e_2\}$}] at (-1.5,-4) {$\textcolor{black}{\mathfrak{n}_{2}}$};
\node (0dot3) [cnodesmall,draw=black,thick,label={below: \small $\{e_3\}$}] at (3,-4) {$\textcolor{black}{\mathfrak{n}_{3}}$};
\node (0dot4) [cnodesmall,draw=black,thick,label={below: \small $\{e_4\}$}] at (7.5,-4) {$\textcolor{black}{\mathfrak{n}_{4}}$};



\node (prdot1) [cnodesmall,draw=black,thick,label={left: \small $\{e_3, e_4\}$}] at (-8,2) {$\textcolor{black}{\mathfrak{n}_{5}}$}; 
\node (prdot2) [cnodesmall,draw=black,thick,label={above: \small $\phantom{L}\{e_1, e_4\}$}] at (0,2) {$\textcolor{black}{\mathfrak{n}_{6}}$};
\node (prdot3) [cnodesmall,draw=black,thick,label={left: \small $\{e_2, e_3\}$}] at (8,2) {$\textcolor{black}{\mathfrak{n}_{7}}$};

\node (trdot1) [cnodesmall,draw=black,thick,label={left: \small $\{e_2, e_3, e_4\}$}] at (-6,8) {$\textcolor{black}{\mathfrak{n}_{8}}$}; 
\node (trdot2) [cnodesmall,draw=black,thick,label={}] at (0,8) {$\textcolor{black}{\mathfrak{n}_{9}}$}; 
\node (trdot3) [cnodesmall,draw=black,thick,label={right: \small $\{e_1, e_2, e_4\}$}] at (6,8) {$\textcolor{black}{\mathfrak{n}_{10}}$}; 

\node (qrdot1) [cnodesmall,draw=black,thick,label={right: \small $\{e_1, e_2, e_3, e_4\}$}] at (0,14) {$\textcolor{black}{ \mathfrak{n}_r}$};

\draw[thick,->, gray,fill=gray] (0dot3) -- (prdot1); \draw[thick,->, gray] (0dot4) -- (prdot1); 
\draw[thick,->, gray] (0dot1) -- (prdot2); \draw[thick,->, gray] (0dot4) -- (prdot2);
\draw[thick,->, gray] (0dot2) -- (prdot3); \draw[thick,->, gray] (0dot3) -- (prdot3); 

\draw[thick,->, gray, fill=gray] (0dot2) -- (trdot1); \draw[thick,->, gray] (prdot1) -- (trdot1);
\draw[thick,->, gray] (0dot2) -- (trdot3); \draw[thick,->, gray] (prdot2) -- (trdot3); 
\draw[thick,->, gray] (0dot1) -- (trdot2); \draw[thick,->, gray] (prdot3) -- (trdot2);

\draw[thick,->, gray] (trdot1) -- (qrdot1); \draw[thick,->, gray] (trdot2) -- (qrdot1);
\draw[thick,->, gray] (trdot3) -- (qrdot1);

\end{tikzpicture}	
\hspace{40pt}
\begin{tikzpicture} [scale=0.23,every node/.style={scale=0.7}, level/.style={},decoration={brace,mirror,amplitude=7}, >={Stealth[gray]}]


\node (0dot1) [cnodesmall,draw=black,thick,label={below: \small $\textcolor{black}{\mathfrak{n}_{1}}$}] at (-6,-4) {$f_1$};
\node (0dot2) [cnodesmall,draw=black,thick,label={below: \small $\textcolor{black}{\mathfrak{n}_{2}}$}] at (-1.5,-4) {$f_2$};
\node (0dot3) [cnodesmall,draw=black,thick,label={below: \small $\textcolor{black}{\mathfrak{n}_{3}}$}] at (3,-4) {$f_3$};
\node (0dot4) [cnodesmall,draw=black,thick,label={below: \small $\textcolor{black}{\mathfrak{n}_{4}}$}] at (7.5,-4) {$f_4$};



\node (prdot1) [cnodesmall,draw=black,thick,label={below: \small $\textcolor{black}{\mathfrak{n}_{5}}$}] at (-8,2) {$f_{5}$}; 
\node (prdot2) [cnodesmall,draw=black,thick,label={below: \small $\textcolor{black}{\mathfrak{n}_{6}}$}] at (0,2) {$f_{6}$};
\node (prdot3) [cnodesmall,draw=black,thick,label={below: \small $\textcolor{black}{\mathfrak{n}_{7}}$}] at (8,2) {$f_{7}$};

\node (trdot1) [cnodesmall,draw=black,thick,label={below left: \small $\textcolor{black}{\mathfrak{n}_{8}}$}] at (-6,8) {$f_8$}; 
\node (trdot2) [cnodesmall,draw=black,thick,label={below left: \small $\textcolor{black}{\mathfrak{n}_{9}}$}] at (0,8) {$f_9$}; 
\node (trdot3) [cnodesmall,draw=black,thick,label={below right: \small $\textcolor{black}{\mathfrak{n}_{10}}$}] at (6,8) {$f_{10}$}; 

\node (qrdot1) [cnodesmall,draw=black,thick,label={right: \small $\textcolor{black}{\mathfrak{n}_{r}}$}] at (0,14) {$\textcolor{black}{f_r}$};

\draw[thick,->, gray] (0dot3) -- (prdot1); \draw[thick,->, gray] (0dot4) -- (prdot1); 
\draw[thick,->, gray] (0dot1) -- (prdot2); \draw[thick,->, gray] (0dot4) -- (prdot2);
\draw[thick,->, gray] (0dot2) -- (prdot3); \draw[thick,->, gray] (0dot3) -- (prdot3); 

\draw[thick,->, gray] (0dot2) -- (trdot1); \draw[thick,->, gray] (prdot1) -- (trdot1);
\draw[thick,->, gray] (0dot2) -- (trdot3); \draw[thick,->, gray] (prdot2) -- (trdot3); 
\draw[thick,->, gray] (0dot1) -- (trdot2); \draw[thick,->, gray] (prdot3) -- (trdot2);

\draw[thick,->, gray] (trdot1) -- (qrdot1); \draw[thick,->, gray] (trdot2) -- (qrdot1);
\draw[thick,->, gray] (trdot3) -- (qrdot1); 
   
\end{tikzpicture}
\end{center}
\caption{\label{fig: parts}
(a) A \emph{composition scheme} for an object \m{\Gcal} is a DAG in which the leaves correspond to atoms, 
the internal nodes correspond to sets of atoms, and the root corresponds to the entire object.  
(b) A \emph{compositional network} is a composition scheme in which each node \m{\neuron_i} 
also carries a feature vector \m{f_i}. The feature vector at \m{\neuron_i} is computed 
from the feature vectors of the children of \m{\neuron_i}. 
}
\end{figure}

%% file: fig-invariant-parts.tex
\begin{figure}[t]
\begin{center}
	
\begin{tikzpicture}[scale = 0.2, every node/.style={scale=0.8},level/.style={},decoration={brace,mirror,amplitude=7}, >={Stealth[black]}]

\node (DAGLabel) at (-6,15) {$\mathcal{M}$};

\node (0dot1) [cnodesmall,draw=black,thick,label={left: \normalsize $e_1$}] at (-6,-4) {};
\node (0dot2) [cnodesmall,draw=black,thick,label={left: \normalsize $e_2$}] at (-1.5,-4) {};
\node (0dot3) [cnodesmall,draw=black,thick,label={left: \normalsize $e_3$}] at (3,-4) {};
\node (0dot4) [cnodesmall,draw=black,thick,label={left: \normalsize $e_4$}] at (7.5,-4) {};



\node (prdot1) [cnodesmall,draw=black,thick,label={}] at (-8,2) {\tiny $$}; 
\node (prdot2) [cnodesmall,draw=black,thick,label={}] at (0,2) {\tiny $$};
\node (prdot3) [cnodesmall,draw=black,thick,label={}] at (8,2) {\tiny $$};

\node (trdot1) [cnodesmall,draw=black,thick,label={}] at (-6,8) {\tiny $$}; 
\node (trdot2) [cnodesmall,draw=black,thick,label={}] at (0,8) {\tiny $$}; 
\node (trdot3) [cnodesmall,draw=black,thick,label={}] at (6,8) {\tiny $$}; 

\node (qrdot1) [cnodesmall,draw=black,thick,label={left: \normalsize $\mathfrak{n}_{r}$}] at (0,14) {};

\draw[thick,->, black] (0dot3) -- (prdot1); \draw[thick,->, black] (0dot4) -- (prdot1); 
\draw[thick,->, black] (0dot1) -- (prdot2); \draw[thick,->, black] (0dot4) -- (prdot2);
\draw[thick,->, black] (0dot2) -- (prdot3); \draw[thick,->, black] (0dot3) -- (prdot3); 

\draw[thick,->, black] (0dot2) -- (trdot1); \draw[thick,->, black] (prdot1) -- (trdot1);
\draw[thick,->, black] (0dot2) -- (trdot3); \draw[thick,->, black] (prdot2) -- (trdot3); 
\draw[thick,->, black] (0dot1) -- (trdot2); \draw[thick,->, black] (prdot3) -- (trdot2);

\draw[thick,->, black] (trdot1) -- (qrdot1); \draw[thick,->, black] (trdot2) -- (qrdot1);
\draw[thick,->, black] (trdot3) -- (qrdot1);


\node (DAGLabel2) at (16,15) {$\mathcal{M'}$};

\node (0dot12) [cnodesmall,draw=black,thick,label={right: \normalsize $e_1$}] at (16,-4) {};
\node (0dot22) [cnodesmall,draw=black,thick,label={right: \normalsize $e_4$}] at (20.5,-4) {};
\node (0dot32) [cnodesmall,draw=black,thick,label={right: \normalsize $e_2$}] at (25,-4) {};
\node (0dot42) [cnodesmall,draw=black,thick,label={right: \normalsize $e_3$}] at (29.5,-4) {};

\draw[thin,dashed, ->, blue, bend right] (0dot1) to [out=-20,in=200] (0dot12);

\draw[thin,dashed, ->, blue, bend right] (0dot2) to [out=-20,in=200] (0dot32);

\draw[thin,dashed, ->, blue, bend right] (0dot3) to [out=20,in=160] (0dot42);

\draw[thin,dashed, ->, blue, bend right] (0dot4) to [out=-20,in=120] (0dot22);



\node (prdot12) [cnodesmall,draw=black,thick,label={}] at (14,2) {\tiny $$}; 
\node (prdot22) [cnodesmall,draw=black,thick,label={}] at (22,2) {\tiny $$};
\node (prdot32) [cnodesmall,draw=black,thick,label={}] at (30,2) {\tiny $$};

\node (trdot12) [cnodesmall,draw=black,thick,label={}] at (16,8) {\tiny $$}; 
\node (trdot22) [cnodesmall,draw=black,thick,label={}] at (22,8) {\tiny $$}; 
\node (trdot32) [cnodesmall,draw=black,thick,label={}] at (28,8) {\tiny $$}; 

\node (qrdot12) [cnodesmall,draw=black,thick,label={right: \normalsize $\mathfrak{n'}_{r}$}] at (22,14) {}; 

\draw[thick,->, black] (0dot12) -- (prdot12); \draw[thick,->, black] (0dot22) -- (prdot12); 
\draw[thick,->, black] (0dot32) -- (prdot22); \draw[thick,->, black] (0dot42) -- (prdot22); 
\draw[thick,->, black] (0dot22) -- (prdot12); \draw[thick,->, black] (0dot42) -- (prdot32); \draw[thick,->, black] (0dot22) -- (prdot32);

\draw[thin,dashed, ->, blue, bend right] (prdot2) to [out=20,in=160] (prdot12);
\draw[thin,dashed, ->, blue, bend right] (prdot3) to [out=20,in=160] (prdot22);

\draw[thick,->, black] (0dot32) -- (trdot32); \draw[thick,->, black] (prdot32) -- (trdot32); 
\draw[thin,dashed, ->, blue, bend right] (trdot1) to [out=20,in=160] (trdot32);

\draw[thick,->, black] (prdot22) -- (trdot22); \draw[thick,->, black] (0dot12) -- (trdot22);
\draw[thin,dashed, ->, blue, bend right] (trdot2) to [out=-20,in=200] (trdot22);

\draw[thick,->, black] (prdot12) -- (trdot12); \draw[thick,->, black] (0dot32) -- (trdot12);
\draw[thin,dashed, ->, blue, bend right] (trdot3) to [out=20,in=160] (trdot12);

\draw[thin,dashed, ->, blue] (qrdot1) -- (qrdot12);

\draw[thick,->, black] (trdot12) -- (qrdot12); \draw[thick,->, black] (trdot22) -- (qrdot12); \draw[thick,->, black] (trdot32) -- (qrdot12);

\end{tikzpicture}
\end{center}
\caption{\label{fig: invariant parts}
A minimal requirement for composition schemes is that they be invariant to permutation, 
i.e.~that if the numbering of the atoms is changed by a permutation \m{\sigma}, 
then we must get an isomorphic DAG. Any node in the new DAG that corresponds 
to \m{\cbrN{e'_{i_1},\ldots,e'_{i_k}}} must have a corrresponding node in the 
old DAG corresponding to \m{\cbrN{e_{\sigma^{-1}(i_1)},\ldots,e_{\sigma^{-1}(i_k)}}}. 
}
\end{figure}
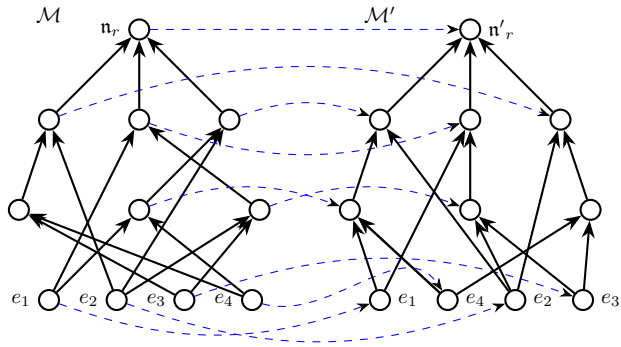

%% file: covariant2.tex
\section{Covariant compositional networks}

One of the messages of the present paper is that invariant message passing algorithms, 
of the form described in the previous section, 
are \emph{not} the most general possible compositional models 
for producing permutation invariant representations of graphs (or of compound objects, in general). 

Once again, an analogy with image recognition is helpful. 
Classical CNNs face two types of basic image transformations: translations and rotations. 
With respect to translations (barring pooling, edge effects and other complications), 
CNNs behave in a quasi-invariant way, in the sense that if the input image is translated by any integer amount  
\m{(t_x,t_y)}, the activations in each layer \m{\ell=1,2,\ldots L} 
translate the same way: 
the activation of any neuron \m{\neuron^\ell_{i,j}} is simply transferred to 
neuron \m{\smash{\neuron^\ell_{i+t_1,j+t_2}}}, 
i.e., \m{\smash{{f'}{}^{\ell}_{i+t_1,j+t_2}\!\!\<=f^\ell_{i,j}}}. 
This is the simplest manifestation of a well studied property of CNNs called \emph{equivariance}  
\citep{CohenWelling2016, Worrall2017}. 

\input{fig-steering} 
With respect to rotations, however, the situation is more complicated: if we rotate the input image 
by, e.g., \m{90} degrees, not only will the part of the image that fell in the 
receptive field of a particular neuron \m{\neuron^\ell_{i,j}} move to the receptive field of a different 
neuron \m{\neuron^{\ell}_{j,-i}}, but the orientation of the receptive field will also change  
(Figure \ref{fig: sterring}). 
Consequently, features which were, for example, previously picked up by horizontal filters will now be picked 
up by vertical filters.  
Therefore, in general, \m{{f'}{}^\ell_{j,-i}\!\neq\! f^\ell_{i,j}}. 
It can be shown that one cannot construct a CNN for images that behaves 
in a quasi-invariant way with respect to both translations and rotations unless every filter  
is directionless.  

It is, however, possible to construct a CNN in which the activations transform in a 
predictable and reversible way,  
in particular, \m{{f'}{}^\ell_{j,-i}\<= R(f^\ell_{i,j})} for some fixed invertible function \m{R}.  
This phenomenon is called \emph{steerability}, and has a significant literature in 
both classical signal processing \citep{FreemanAdelson1991,  Simocelli1992, Perona1995, Teo1998, Manduchi1998} and the neural networks field \citep{CohenWelling2017}. 

The situation in compositional networks is similar. 
The comp-net and message passing architectures that we have examined so 
far, by virtue of the aggregation function being symmetric in its arguments, are all  
\emph{quasi-invariant} (with respect to permutations) in the following sense.

\begin{definition}\label{def: invariant compnet}
Let \m{\Gcal} be a compound object of \m{n} parts and \m{\Gcal'} an equivalent object 
in which the atoms have been permuted by some permutation \m{\sigma}. 
Let \m{\Ncal} be a comp-net for \m{\Gcal} based on an invariant composition scheme, 
and \m{\Ncal'} be the corresponding network for \m{\Gcal'}. 
We say that \m{\Ncal} is \df{quasi-invariant} if for any \m{\neuron_i\tin \Ncal}, 
letting \m{\neuron'_j} be the corresponding node in \m{\Ncal'}, \m{f_i\<=f'_j} 
for any \m{\sigma\tin\Sn} 
\end{definition}

Quasi-invariance in comp-nets is equivalent to the assertion that the activation 
\m{f_i} at any given node must only depend on \m{\prt_i=\cbrN{e_{j_1},\ldots,e_{j_k}}} as a \emph{set}, 
and not on the internal ordering of the atoms \m{e_{j_1},\ldots,e_{j_k}} making up the receptive field. 
At first sight this seems desirable, since it is exactly what we expect from the overall representation \m{\phi(G)}. 
On closer examination, however, we realize that this property is potentially problematic, since it 
means that \m{\neuron_i} has 
lost all information about which vertex in its receptive field 
has contributed what to the aggregate information \m{f_i}. 
In the CNN analogy, we can say that we have lost information about the \emph{orientation} of the receptive field.  
In particular, if, further upstream, \m{f_i} is combined with some other   
feature vector \m{f_j} from a node with an overlapping receptive field, 
the aggregation process has no way of taking into account which parts of the information in 
\m{f_i} and \m{f_j} come from shared vertices and which parts do not (Figure \ref{fig: aggregation}). 
\input{fig-aggregation} 

The solution is to upgrade the \m{\prt_i} receptive fields to be \emph{ordered sets}, 
and explicitly establish how \m{f_i} co-varies with the internal ordering of the receptive fields.  
To emphasize that henceforth the \m{\prt_i} sets are ordered, 
we will use parentheses rather than braces to denote their content. 

\begin{definition}\label{def: steerable compnet} 
Let \m{\Gcal}, \m{\Gcal'}, \m{\Ncal} and \m{\Ncal'} be as in Definition \ref{def: invariant compnet}. 
Let \m{\neuron_i} be any node of \m{\Ncal} and \m{\neuron_j} the corresponding node of \m{\Ncal'}. 
Assume that \m{\prt_i=\brN{e_{p_1},\ldots,e_{p_m}}} while \m{\prt'_j=\brN{e_{q_1},\ldots,e_{q_m}}}, 
and let \m{\pi\tin\Sbb_m} be the permutation that 
aligns the orderings of the two receptive fields, i.e., for which \m{e_{q_{\pi(a)}}\!\<=e_{p_a}}.  
We say that \m{\Ncal} is \df{covariant to permutations} if for any \m{\pi}, there is a 
corresponding function \m{R_\pi} such that \m{f'_j=R_{\pi}(f_i)}.  
\end{definition}

\subsection{First order covariant comp-nets}

The form of covariance prescribed by Definition \ref{def: steerable compnet} is very general. 
To make it more specific, in line with the classical literature on steerable representations, 
we make the assumption that the \m{\cbrN{f\mapsto R_\pi(f)}_{\pi\in\Sbb_m}} maps are  \emph{linear}, 
and by abuse of notation, from now on simply treat them as matrices (with \m{R_\pi(f)=R_\pi f}).  
The linearity assumption automatically implies that \m{\cbrN{R_\pi}_{\pi\in \Sbb_m}} 
is a \emph{representation} of \m{\Sbb_m} in the group theoretic 
sense of the word  (for the definition of group representations, see the Appendix)\footnote{
This notion of representation must not be confused with the neural networks sense of representations 
of objects, as in ``\m{f^\ell_i} is a representation of \m{\Pcal^\ell_i}''}. 

\begin{proposition}\label{prop: steerability} 
If for any \m{\pi\tin\Sbb_m}, the \m{f\mapsto R_\pi(f)} map appearing in Definition \ref{def: steerable compnet} 
is linear, then the corresponding \m{\cbrN{R_\pi}_{\pi\in\Sbb_m}} matrices form a representation of \m{\Sbb_m}. 
\end{proposition} 

The representation theory of symmetric groups is a rich subject that goes beyond the scope of the present 
paper \citep{Sagan}. 
However, there is one particular representation of \m{\Sbb_m} that is likely familiar even to non-algebraists, the 
so-called \emph{defining representation}, given by the \m{P_\pi\tin\RR^{n\times n}} 
permutation matrices 
\[[P_\pi]_{i,j}=
\begin{cases}
~1&\Tif \ {\pi(j)\<=i}\\
~0&\Totherwise. 
\end{cases}
\]
It is easy to verify that 
\m{P_{\pi_2\pi_1}\<=P_{\pi_2} P_{\pi_1}} for any \m{\pi_1,\nts\pi_2\tin\Sbb_m}, so \m{\cbrN{P_\pi}_{\pi\in\Sbb_m}} 
is indeed a representation of \m{\Sbb_m}.   
If the transformation rules of the \m{f_i} activations in a given comp-net are dictated by 
this representation, then each \m{f_i} must necessarily be a \m{\absN{\Pcal_i}} dimensional vector, 
and intuitively each component of \m{f_i} carries information related to one specific 
atom in the receptive field, or the interaction of that specific atom with all the others.  
We call this case \df{first order permutation covariance}.  

\begin{definition}\label{def: 1st order}
We say that \m{\neuron_i} is a \df{first order covariant node} in a comp-net if under the permutation of 
its receptive field \m{\prt_i} by any \m{\pi\tin\Sbb_{\abs{\nts\prt_i\nts}}}, 
its activation trasforms as \m{f_i\mapsto P_\pi f_i}. 
\end{definition}


\subsection{Second order covariant comp-nets}\label{sec: second order}

It is easy to verify that given any representation \m{\smash{\brN{R_g}_{g\in \mathfrak{G}}}} of a group \m{\mathfrak{G}},  
the matrices \m{\brN{R_g\<\otimes R_g}_{g\in\mathfrak{G}}} also furnish a representation of \m{\mathfrak{G}}. 
Thus, one step up in the hierarchy from \m{P_\pi}--covariant comp-nets are 
\m{P_\pi\<\otimes P_\pi}--covariant comp-nets, where the \m{f_i} feature vectors are now 
\m{\abs{\prt_i}^2} dimensional vectors that transform under permutations of the internal ordering 
by \m{\pi} as \m{f_i\mapsto (P_\pi\<\otimes P_\pi) f_i}. 

If we reshape \m{f_i} into a matrix 
\m{F_i\tin\RR^{\absN{\prt_i}\times\absN{\prt_i}}}, then the action 
\begin{equation*}\label{eq: 2nd action}
F_i\mapsto P_\pi\ts F_i\ts P_\pi^\top 
\end{equation*} 
is equivalent to \m{P_\pi\<\otimes P_\pi} acting on \m{f_i}. In the following, we will prefer this 
more intuitive matrix view, since it clearly expresses that feature vectors that transform this way 
express \emph{relationships} between the different constituents of the receptive field. 
Note, in particular, that if we define \m{A\tdown_{\prt_i}} as the restriction of the adjacency matrix to 
\m{\prt_i} (i.e., if \m{\prt_i=\brN{e_{p_1},\ldots,e_{p_m}}} then \m{[A\tdown_{\prt_i}]_{a,b}=A_{p_a,p_b}}), 
then \m{A\tdown_{\prt_i}} transforms exactly as \m{F_i} does in the equation above. 

\begin{definition}\label{def: 2nd order}
We say that \m{\neuron_i} is a \df{second order covariant node} in a comp-net 
if under the permutation of its receptive field \m{\prt_i} by any \m{\pi\tin\Sbb_{\abs{\nts\prt_i\nts}}},  
its activation transforms as \m{F_i\mapsto P_\pi\ts F_i\ts P_\pi^\top}. 
\end{definition}

\subsection{Third and higher order covariant comp-nets}

Taking the pattern further lets us consider third, fourth, and general, \m{k}'th order nodes in our comp-net, 
in which the activations are \m{k}'th order tensors, transforming under permutations as 
\begin{equation*}%
F_i\mapsto F_i'\qquad \textrm{where}\qquad 
[F'_i]_{\sseq{j}{k}}=
\sum_{j'_1}\sum_{j'_2}\ldots \sum_{j'_k} [P_\pi]_{j_1,j'_1} [P_\pi]_{j_2,j'_2}  \ldots [P_\pi]_{j_k,j'_k} 
[F_i]_{\sseq{j'}{k}},  
\end{equation*}
In the more compact, so called Einstein notation\footnote{The Einstein convention is that if, in a 
given tensor expression the same index appears twice, once ``upstairs'' and once ``downstairs'', then 
it is summed over. For example, the matrix/vector product \m{y\<=Ax} would be written \m{y_i\<=A_i^{\;j} x_j}}, 
\begin{equation}\label{eq: kth action}
[F'_i]_{\sseq{j}{k}}=[P_\pi]^{\phantom{A}j'_1}_{j_1}\, [P_\pi]^{\phantom{A}j'_2}_{j_2}\:  \ldots\: 
[P_\pi]^{\phantom{A}j'_k}_{j_k}\:[F_i]_{\sseq{j'}{k}}.  
\end{equation}
In general, we will call any quantity which transforms according to this equation a \df{k'th order P-tensor}. 
Note that this notion of tensors is distinct from the common usage of the term in neural networks, 
and more similar to how the word is used in Physics, 
because it not only implies that \m{F_i} is a quanity representable by an \m{m\<\times m\<\times\ldots\times m} 
array of numbers, but also that \m{F_i} transforms in a specific way. 

Since scalars, vectors and matrices can be considered as \m{0^\text{th}}, \m{1^\text{st}} and \m{2^\text{nd}} order 
tensors, respectively, 
the following definition covers Definitions \ref{def: invariant compnet}, \ref{def: 1st order} 
and \ref{def: 2nd order} as special cases (with quasi-invariance being equivalent to zeroth order equivariance). 
To unify notation and terminology, regardless of the dimensionality, 
in the following we will always talk about \emph{feature tensors} rather than \emph{feature vectors},   
and denote the activations with \m{F_i} rather than \m{f_i}, as we did in the first half of the paper. 

\begin{definition}
We say that \m{\neuron_i} is a \df{k'th order covariant node} in a comp-net   
if the corresponding activation \m{F_i} is a \m{k}'th order \m{P}--tensor, i.e., it 
transforms under permutations of \m{\prt_i} according to \rf{eq: kth action}, 
or the activation is a sequence of \m{c} separate \m{P}--tensors \m{\smash{F_i^{(1)},\ldots, F_i^{(c)}}} 
corresponding to \m{c} distinct channels. 
\end{definition}

%% file: fig-steering.tex
\begin{figure}[t]
\vspace{-20pt}
\begin{center}
\includepic{.4}{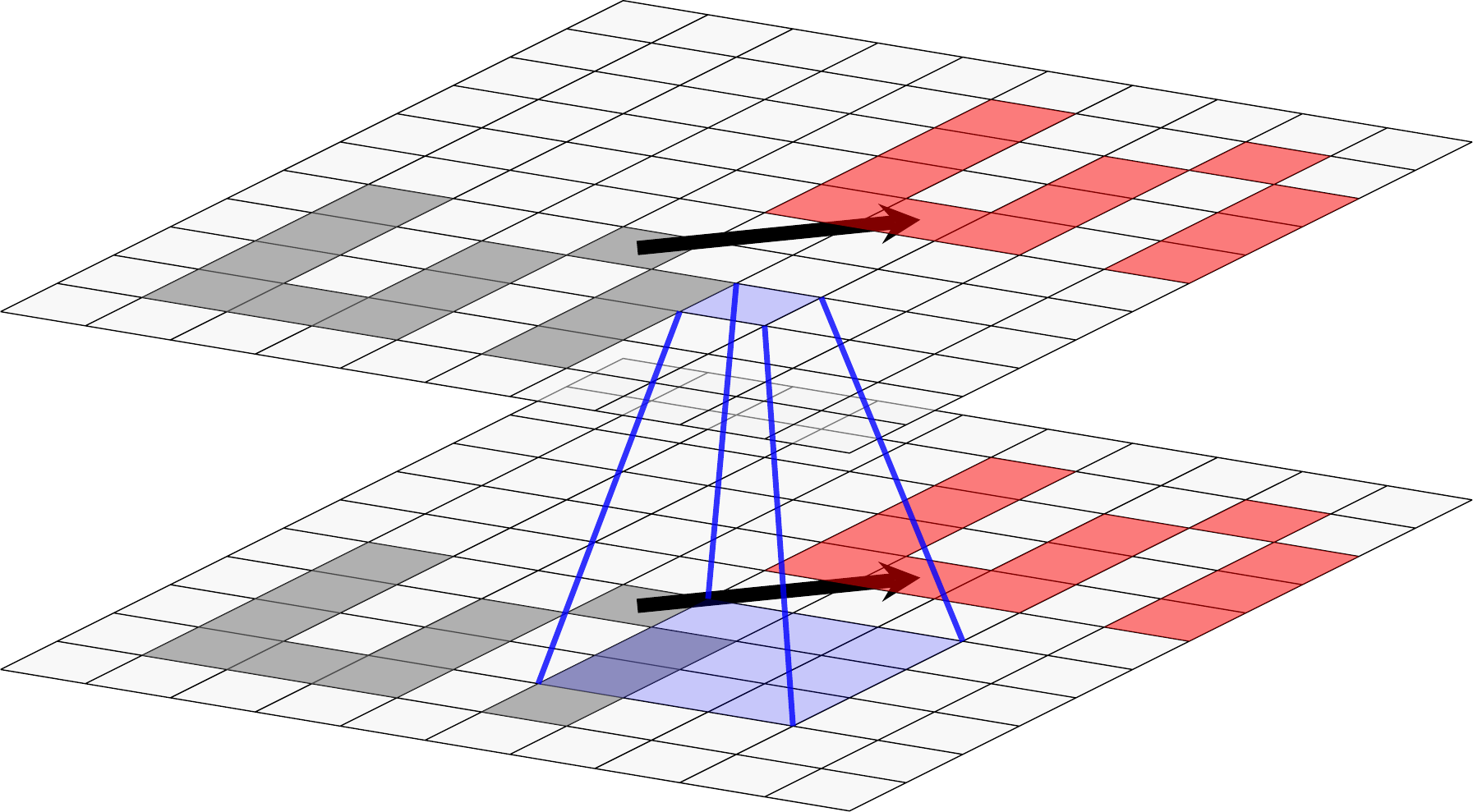}
\includepic{.4}{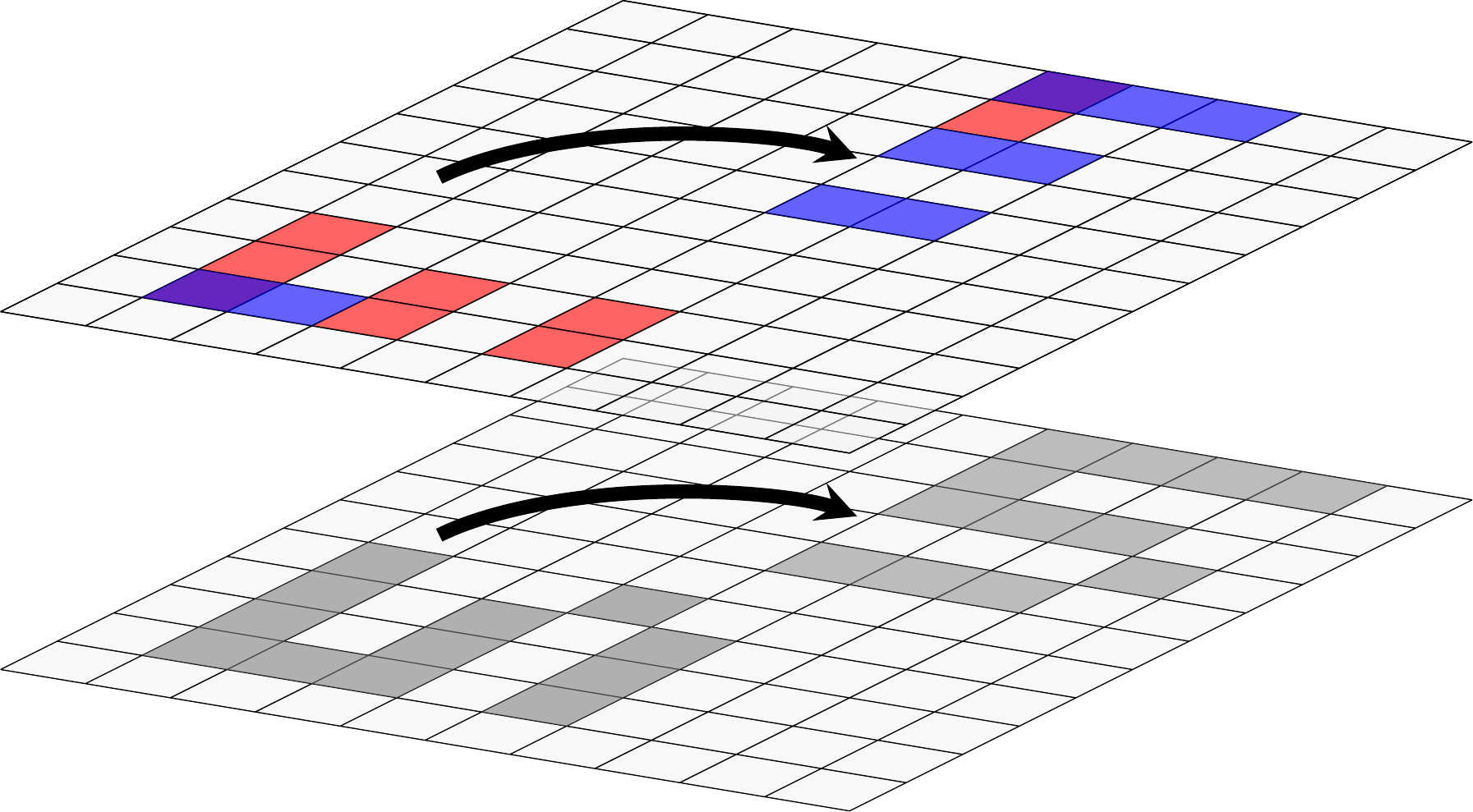}
\end{center}
\caption{\label{fig: sterring} 
In convolutional neural networks if the input image is translated by some amount \m{(t_1,t_2)}, what 
used to fall in the receptive field of neuron \m{\neuron^\ell_{i,j}} is moved to the 
receptive field of \m{\smash{\neuron^\ell_{i+t_1,j+t_2}}}. 
Therefore, the activations transform in the very simple way \m{\smash{f'{}^\ell_{i+t_1,j+t_2}=f^\ell_{i,j}}}. 
In contrast, rotations not only move the receptive fields around, but also permute the neurons in 
the receptive field internally, therefore, in general, \m{\smash{{f'}{}^\ell_{j,-i}\!\neq\! f^\ell_{i,j}}}. 
The right hand figure shows that if the CNN has a horizontal filter (blue) 
and a vertical one (red) then their activations are exchanged by a 90 degree rotation. 
In steerable CNNs,~ if \m{(i,j)\mapsto (i',j')}, then \m{{f'}{}^\ell_{i',j'}\<= R(f^\ell_{i,j})} for some fixed linear function of the 
rotation. 
}
\end{figure}

%% file: fig-aggregation.tex
\begin{figure}[t]
\begin{center}
\begin{tikzpicture}[scale=.5]
\definecolor{light-gray}{gray}{0.95}

\node[circle,draw=black, fill=red, inner sep=0pt,minimum size=5pt,label={below: \tiny $1$}] (root1) at (0,0) {};
\node[circle,draw=black, fill=red, inner sep=0pt,minimum size=5pt,label={below: \tiny $2$}] (level1node1) at (2,-2) {};
\node[circle,draw=black, fill=red, inner sep=0pt,minimum size=5pt,label={below: \tiny $3$}] (level1node2) at (-2,-2) {};

\node[circle,draw=black, fill=red, inner sep=0pt,minimum size=5pt,label={below: \tiny $4$}] (level2node1) at (-3.5,-4) {};
\node[circle,draw=black, fill=red, inner sep=0pt,minimum size=5pt,label={below: \tiny $5$}] (level2node2) at (-0.5,-4) {};

\node[circle,draw=black, fill=red, inner sep=0pt,minimum size=5pt,label={below: \tiny $6$}] (level2node3) at (3.5,-4) {};
\node[circle,draw=black, fill=red, inner sep=0pt,minimum size=5pt,label={below: \tiny $7$}] (level2node4) at (0.5,-4) {};

\draw[black] (root1) -- (level1node1) (root1) -- (level1node2) (level1node1) -- (level2node3) (level1node1) -- (level2node4) (level1node2) -- (level2node1) (level1node2) -- (level2node2);

\node[circle,draw=black, fill=red, inner sep=0pt,minimum size=5pt,label={below: \tiny $1$}] (root1) at (0,6) {};
\node[circle,draw=black, fill=red, inner sep=0pt,minimum size=5pt,label={below: \tiny $2$}] (level1node1) at (2,4) {};
\node[circle,draw=black, fill=red, inner sep=0pt,minimum size=5pt,label={below: \tiny $3$}] (level1node2) at (-2,4) {};

\node[circle,draw=black, fill=red, inner sep=0pt,minimum size=5pt,label={below: \tiny $4$}] (level2node1) at (-3.5,2) {};
\node[circle,draw=black, fill=red, inner sep=0pt,minimum size=5pt,label={below: \tiny $5$}] (level2node2) at (0,2) {};

\node[circle,draw=black, fill=red, inner sep=0pt,minimum size=5pt,label={below: \tiny $6$}] (level2node3) at (3.5,2) {};

\draw[black] (root1) -- (level1node1) (root1) -- (level1node2) (level1node1) -- (level2node3) (level1node1) -- (level2node2) (level1node2) -- (level2node1) (level1node2) -- (level2node2);

\draw[thick,-latex,blue] (-3.5,-3.5) -- (-2.3,-2);
\draw[thick,-latex,blue] (-0.5,-3.5) -- (-1.7,-2);

\draw[thick,-latex,blue] (0.5,-3.5) -- (1.7,-2);
\draw[thick,-latex,blue] (3.5,-3.5) -- (2.3,-2);

\draw[thick,-latex,blue] (-3.5,2.5) -- (-2.3,4);
\draw[thick,-latex,blue] (-0.1,2.5) -- (-1.7,4);

\draw[thick,-latex,blue] (0.1,2.5) -- (1.7,4);
\draw[thick,-latex,blue] (3.5,2.5) -- (2.3,4);

\draw[thick,-latex,blue] (2,-1.5) -- (0.35,0);
\draw[thick,-latex,blue] (-2,-1.5) -- (-0.35,0);

\draw[thick,-latex,olive] (1.45,-1.9) -- (0.2,-0.6);
\draw[thick,-latex,olive] (-1.45,-1.9) -- (-0.2,-0.6);

\draw[thick,-latex,blue] (2,4.5) -- (0.35,6);
\draw[thick,-latex,blue] (-2,4.5) -- (-0.35,6);

\draw[thick,-latex,olive] (1.45,4.1) -- (0.2,5.4);
\draw[thick,-latex,olive] (-1.45,4.1) -- (-0.2,5.4);

\end{tikzpicture}
\hspace{30pt}
\begin{tikzpicture}[scale=.3,every node/.style={minimum size=1cm},on grid]

\begin{scope}[  
yshift=20,every node/.append style={
	yslant=0.5,xslant=-1.5},yslant=0.5,xslant=-1.5, rotate=270
]
\definecolor{light-gray}{gray}{0.95}
\fill[light-gray,fill opacity=0.5] (0,0) rectangle (9.5,3.5);
\draw[step=3mm, gray, opacity=0.05] (0,0) grid (4.5,3.5);

\node[circle,draw=black, fill=red, inner sep=0pt,minimum size=3pt] (A0) at (0.2,0.2) {};
\node[circle,draw=black, fill=red, inner sep=0pt,minimum size=3pt] (B0) at (0.2,3.2) {};

\node[circle,draw=black, fill=red, inner sep=0pt,minimum size=3pt] (C0) at (2.7,1.6) {};

\node[circle,draw=black, fill=red, inner sep=0pt,minimum size=3pt] (D0) at (5.2,0.2) {};
\node[circle,draw=black, fill=red, inner sep=0pt,minimum size=3pt] (E0) at (5.2,3.2) {};

\node[circle,draw=black, fill=red, inner sep=0pt,minimum size=3pt] (F0) at (6.7,1.6) {};      

\node[circle,draw=black, fill=red, inner sep=0pt,minimum size=3pt] (G0) at (8.2,2.5) {};
\node[circle,draw=black, fill=red, inner sep=0pt,minimum size=3pt] (H0) at (9.3,1) {};

\node[circle,draw=black, fill=red, inner sep=0pt,minimum size=3pt] (I0) at (9.3,3.2) {};

\draw[black] (A0) -- (C0) (B0) -- (C0) (D0) -- (C0) (C0) -- (E0) (D0) -- (E0) (D0) -- (F0) (E0) -- (F0) (F0) -- (G0) (G0) -- (H0) (G0) -- (I0);



\draw[blue, thick, rotate around={42:(G0)}] (G0) ellipse (1.4cm and 2cm); 

\draw[blue, thick, rotate around={42:(D0)}] (5.2,0.8) ellipse (3.2cm and 3.2cm); 

\draw[blue, thick, rotate=45] (C0) ellipse (3.2cm and 3.2cm); 

\end{scope}

\begin{scope}[  
yshift=140,every node/.append style={
	yslant=0.5,xslant=-1.5},yslant=0.5,xslant=-1.5, rotate=270
]
\definecolor{light-gray}{gray}{0.95}
\fill[light-gray,fill opacity=0.5] (0,0) rectangle (9.5,3.5);
\draw[step=3mm, gray, opacity=0.05] (0,0) grid (4.5,3.5);

\node[circle,draw=black, fill=blue, inner sep=0pt,minimum size=3pt] (A1) at (0.2,0.2) {};
\node[circle,draw=black, fill=blue, inner sep=0pt,minimum size=3pt] (B1) at (0.2,3.2) {};

\node[circle,draw=black, fill=blue, inner sep=0pt,minimum size=3pt] (C1) at (2.7,1.6) {};

\node[circle,draw=black, fill=blue, inner sep=0pt,minimum size=3pt] (D1) at (5.2,0.2) {};
\node[circle,draw=black, fill=blue, inner sep=0pt,minimum size=3pt] (E1) at (5.2,3.2) {};

\node[circle,draw=black, fill=blue, inner sep=0pt,minimum size=3pt] (F1) at (6.7,1.6) {};      

\node[circle,draw=black, fill=blue, inner sep=0pt,minimum size=3pt] (G1) at (8.2,2.5) {};
\node[circle,draw=black, fill=blue, inner sep=0pt,minimum size=3pt] (H1) at (9.3,1) {};

\node[circle,draw=black, fill=blue, inner sep=0pt,minimum size=3pt] (I1) at (9.3,3.2) {};

\draw[black] (A1) -- (C1) (B1) -- (C1) (D1) -- (C1) (C1) -- (E1) (D1) -- (E1) (D1) -- (F1) (E1) -- (F1) (F1) -- (G1) (G1) -- (H1) (G1) -- (I1);

\draw[green, thick, rotate=45] (C1) ellipse (3.2cm and 3.2cm); 

\end{scope}

\begin{scope}[  
yshift=260,every node/.append style={
	yslant=0.5,xslant=-1.5},yslant=0.5,xslant=-1.5, rotate=270
]
\definecolor{light-gray}{gray}{0.95}
\fill[light-gray,fill opacity=0.5] (0,0) rectangle (9.5,3.5);
\draw[step=3mm, gray, opacity=0.05] (0,0) grid (4.5,3.5);

\node[circle,draw=black, fill=green, inner sep=0pt,minimum size=3pt] (A2) at (0.2,0.2) {};
\node[circle,draw=black, fill=green, inner sep=0pt,minimum size=3pt] (B2) at (0.2,3.2) {};

\node[circle,draw=black, fill=green, inner sep=0pt,minimum size=3pt] (C2) at (2.7,1.6) {};

\node[circle,draw=black, fill=green, inner sep=0pt,minimum size=3pt] (D2) at (5.2,0.2) {};
\node[circle,draw=black, fill=green, inner sep=0pt,minimum size=3pt] (E2) at (5.2,3.2) {};

\node[circle,draw=black, fill=green, inner sep=0pt,minimum size=3pt] (F2) at (6.7,1.6) {};      

\node[circle,draw=black, fill=green, inner sep=0pt,minimum size=3pt] (G2) at (8.2,2.5) {};
\node[circle,draw=black, fill=green, inner sep=0pt,minimum size=3pt] (H2) at (9.3,1) {};

\node[circle,draw=black, fill=green, inner sep=0pt,minimum size=3pt] (I2) at (9.3,3.2) {};

\draw[black] (A2) -- (C2) (B2) -- (C2) (D2) -- (C2) (C2) -- (E2) (D2) -- (E2) (D2) -- (F2) (E2) -- (F2) (F2) -- (G2) (G2) -- (H2) (G2) -- (I2);

\draw[cyan, thick] (4.9,3.2) ellipse (3cm and 3cm); 

\end{scope}

\begin{scope}[  
yshift=380,every node/.append style={
	yslant=0.5,xslant=-1.5},yslant=0.5,xslant=-1.5, rotate=270
]
\definecolor{light-gray}{gray}{0.95}
\fill[light-gray,fill opacity=0.5] (0,0) rectangle (9.5,3.5);
\draw[step=3mm, gray, opacity=0.05] (0,0) grid (4.5,3.5);

\node[circle,draw=black, fill=cyan, inner sep=0pt,minimum size=3pt] (A3) at (0.2,0.2) {};
\node[circle,draw=black, fill=cyan, inner sep=0pt,minimum size=3pt] (B3) at (0.2,3.2) {};

\node[circle,draw=black, fill=cyan, inner sep=0pt,minimum size=3pt] (C3) at (2.7,1.6) {};

\node[circle,draw=black, fill=cyan, inner sep=0pt,minimum size=3pt] (D3) at (5.2,0.2) {};
\node[circle,draw=black, fill=cyan, inner sep=0pt,minimum size=3pt] (E3) at (5.2,3.2) {};

\node[circle,draw=black, fill=cyan, inner sep=0pt,minimum size=3pt] (F3) at (6.7,1.6) {};      

\node[circle,draw=black, fill=cyan, inner sep=0pt,minimum size=3pt] (G3) at (8.2,2.5) {};
\node[circle,draw=black, fill=cyan, inner sep=0pt,minimum size=3pt] (H3) at (9.3,1) {};

\node[circle,draw=black, fill=cyan, inner sep=0pt,minimum size=3pt] (I3) at (9.3,3.2) {};

\draw[black] (A3) -- (C3) (B3) -- (C3) (D3) -- (C3) (C3) -- (E3) (D3) -- (E3) (D3) -- (F3) (E3) -- (F3) (F3) -- (G3) (G3) -- (H3) (G3) -- (I3);

\draw[magenta, thick] (4.75,1.75) ellipse (5.3cm and 3.2cm);
\end{scope}

\begin{scope}[  
yshift=500,every node/.append style={
	yslant=0.5,xslant=-1.5},yslant=0.5,xslant=-1.5, rotate=270
]
\definecolor{light-gray}{gray}{0.95}
\fill[light-gray,fill opacity=0.0] (0,0) rectangle (6.5,3.5);
\draw[step=3mm, gray, opacity=0.0] (0,0) grid (4.5,3.5);

\node[circle,draw=black, fill=magenta, inner sep=0pt,minimum size=3pt] (Aglobal) at (4.75,1.5) {};

\end{scope}

\draw[middlearrow={latex}, magenta] (A3) -- (Aglobal);
\draw[middlearrow={latex}, magenta] (B3) -- (Aglobal);
\draw[middlearrow={latex}, magenta] (C3) -- (Aglobal);
\draw[middlearrow={latex}, magenta] (D3) -- (Aglobal);
\draw[middlearrow={latex}, magenta] (E3) -- (Aglobal);
\draw[middlearrow={latex}, magenta] (F3) -- (Aglobal);
\draw[middlearrow={latex}, magenta] (G3) -- (Aglobal);
\draw[middlearrow={latex}, magenta] (H3) -- (Aglobal);
\draw[middlearrow={latex}, magenta] (I3) -- (Aglobal);

\draw[middlearrow={latex}, blue] (D0) -- (D1);
\draw[middlearrow={latex}, blue] (C0) -- (D1);
\draw[middlearrow={latex}, blue] (E0) -- (D1);
\draw[middlearrow={latex}, blue] (F0) -- (D1);
\draw[middlearrow={latex}, bend left, blue] (G0) -- (G1);
\draw[middlearrow={latex}, bend right, blue] (H0) -- (G1);
\draw[middlearrow={latex}, blue] (I0) -- (G1);

\draw[middlearrow={latex}, blue] (A0) -- (C1);
\draw[middlearrow={latex}, blue] (B0) -- (C1);
\draw[middlearrow={latex}, blue] (C0) -- (C1);
\draw[middlearrow={latex}, blue] (D0) -- (C1);
\draw[middlearrow={latex}, blue] (E0) -- (C1);

\draw[middlearrow={latex}, green] (A1) -- (C2);
\draw[middlearrow={latex}, green] (B1) -- (C2);
\draw[middlearrow={latex}, green] (C1) -- (C2);
\draw[middlearrow={latex}, green] (D1) -- (C2);
\draw[middlearrow={latex}, green] (E1) -- (C2);

\draw[middlearrow={latex}, cyan] (C2) -- (E3);
\draw[middlearrow={latex}, cyan] (D2) -- (E3);
\draw[middlearrow={latex}, cyan] (E2) -- (E3);
\draw[middlearrow={latex}, cyan] (F2) -- (E3);

\end{tikzpicture}   
\end{center}
\caption{\label{fig: aggregation}
\textbf{Top left:} At level \m{\ell\<=1} \m{\neuron_3} aggregates information from \m{\cbr{\neuron_4,\neuron_5}} and 
\m{\neuron_2} aggregates information \m{\cbrN{\neuron_5,\neuron_6}}. 
At \m{\ell\<=2}, \m{\neuron_1} collects this summary information from \m{\neuron_3} and \m{\neuron_2}. 
\textbf{Bottom left:}
This graph is not isomorphic to the top one, but the activations of \m{\neuron_3} and \m{\neuron_2} 
at  \m{\ell\<=1} will be identical. Therefore, at \m{\ell\<=2}, \m{\neuron_1} will get the same 
inputs from its neighbors, irrespective of whether or not \m{\neuron_5} and \m{\neuron_7} 
are the same node or not.  
\textbf{Right:} 
Aggregation at different levels. For keeping the figure legible only the neighborhood around one node in 
higher levels is marked. 
}
\end{figure}
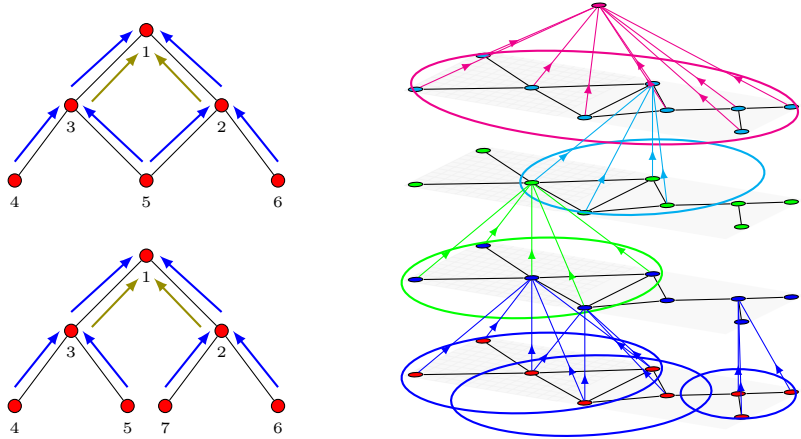

%% file: operations.tex
\section{Tensor aggregation rules}

The previous sections prescribed how activations must transform in comp-nets of different orders, 
but did not explain how this can be assurred, and what it entails for the \m{\Phi} 
aggregation functions. 
Fortunately, tensor arithmetic provides a compact framework for deriving the general 
form of these operations. 
Recall the four basic operations that can be applied to tensors\footnote{
Here and in the following \m{\Tcal^k} will denote the class of \m{k}'th order tensors (\m{k} dimensional tensors), 
regardless of their transformation properties. 
}:
\begin{compactenum}[~1.]
\item The \df{tensor product} of \m{A\tin\Tcal^k} with \m{B\tin\Tcal^p} yields a tensor\,  
\m{C\<=A\<\otimes B\tin \Tcal^{p+k}}\, where  
\[C_{\seq{i}{k+p}}=A_{\seq{i}{k}}\: B_{i_{k+1},i_{k+2},\ldots,i_{k+p}}.\]
\item The \df{elementwise product} of \m{A\tin\Tcal^k} with \m{B\tin\Tcal^p} along 
dimensions \m{(\seq{a}{p})} yields a tensor\,  
\m{\smash{C\<=A\<\odot_{(\sseq{a}{p})} B\tin \Tcal^{k}}}\, where  
\[C_{\seq{i}{k}}=A_{\seq{i}{k}}\: B_{i_{a_1},i_{a_2},\ldots,i_{a_p}}.\]
\item The \df{projection (summation)} of \m{A\tin\Tcal^k} along dimensions 
\m{\cbrN{\seq{a}{p}}} yields a tensor 
\m{C\<=A\tdown_{\sseq{a}{p}}\tin\Tcal^{k-p}} with 
\[C_{\seq{i}{k}}=\sum_{i_{a_1}} \sum_{i_{a_2}}\ldots  \sum_{i_{a_p}} A_{\seq{i}{k}},\]
where we assume that \m{i_{a_1},\ldots,i_{a_p}} have been removed from amongst the indices of \m{C}.
\item The \df{contraction} of \m{A\tin\Tcal^k} along the pair of dimensions \m{\cbrN{a,b}} 
(assuming \m{a\<<b)} yields a \m{k\<-2} order tensor 
\begin{equation*}
C_{\seq{i}{k}}=
\sum_j A_{i_1,\ldots,i_{a-1},j,i_{a+i},\ldots,i_{b-1},j,i_{b+1},\ldots,k},
\end{equation*}
where again we assume that \m{i_a} and \m{i_b} have been removed from amongst the indices of \m{C}. 
Using Einstein notation this can be written much more compactly as 
\[C_{\seq{i}{k}}=
A_{\seq{i}{k}} \delta^{i_a,i_b},
\]
where \m{\delta^{i_a,i_b}} is the diagonal tensor with \m{\delta^{i,j}\<=1} if \m{i\<=j} and \m{0} otherwise. 
In a somewhat unorthodox fashion, we also generalize contractions to (combinations of) larger sets of indices  
\m{\cbrN{\cbrN{\sseq{a^1}{p_1}},\cbrN{\sseq{a^2}{p_2}},\ldots,\cbrN{\sseq{a^q}{p_q}}}} as 
the \m{(k\<-\sum_{j}p_j)} order tensor 
\[C_{\ldots}=A_{\seq{i}{k}}\, \delta^{\sseq{a^1}{p_1}}\, \delta^{\sseq{a^2}{p_2}}\,\ldots\,
\delta^{\sseq{a^q}{p_q}}.
\]
Note that this subsumes projections, since it allows us to write \m{A\tdown_{\sseq{a}{p}}} in the slightly 
unusual looking form 
\[A\tdown_{\sseq{a}{p}}=A_{\seq{i}{k}}\, \delta^{i_{a_1}}\,\delta^{i_{a_2}}\,\ldots\,\delta^{i_{a_k}}.\] 
\end{compactenum}

The following proposition shows that, remarkably, all of the above operations (as well as taking linear 
conbinations) preserve the way that 
\m{P}--tensors behave under permutations and thus they can be freely ``mixed and matched'' within \m{\Phi}.

\begin{proposition}\label{prop: contractions}
Assume that \m{A} and \m{B} are \m{k}'th and \m{p}'th order \m{P}--tensors, respectively. Then
\begin{compactenum}[~~1.]
\item \m{A\< \otimes B}\; is a \m{k\<+p}'th order \m{P}--tensor. 
\item \m{\smash{A\<\odot_{(\sseq{a}{p})} B}}\; is a \m{k}'th order \m{P}--tensor. 
\item \m{A\tdown_{\sseq{a}{p}}}\; is a \m{k\<-p}'th order \m{P}--tensor.
\item \m{\smash{A_{\seq{i}{k}}\, \delta^{\sseq{a^1}{p_1}}\ldots\delta^{\sseq{a^q}{p_q}}}}\; is a 
\m{k\<-\sum_j p_j}'th order \m{P}--tensor.
\end{compactenum}
In addition, if \m{\sseq{A}{u}} are \m{k}'th order \m{P}--tensors and \m{\sseq{\alpha}{u}} are scalars, then   
\m{\smash{\sum_j \alpha_j A_j}} is a \m{k}'th order \m{P}--tensor. 
\end{proposition}

The more challenging part of constructing the aggregation scheme for comp-nets is establishing how 
to relate \m{P}--tensors at different nodes. The following two propositions answer this question. 

\begin{proposition}\label{prop: promotion}
Assume that node \m{\neuron_a} is a descendant of node \m{\neuron_b} in a comp-net \m{\Ncal}, 
\m{\prt_a=(e_{p_1},\ldots,e_{p_{m}})} and \m{\prt_b=(e_{q_1},\ldots,e_{q_{m'}})} are the corresponding 
ordered receptive fields (note that this implies that, as sets, \m{\prt_a\<\subseteq\prt_b}), 
and \m{\chi^{a\to b}\tin\RR^{m\times m'}} is an indicator matrix defined 
\[\chi^{a\to b}_{i,j}=
\begin{cases}
~1&\Tif~~ q_j=p_i\\
~0&\Totherwise. 
\end{cases}\]
Assume that \m{F} is a \m{k}'th order \m{P}--tensor with respect to permutations of 
\m{(e_{p_1},\ldots,e_{p_{m}})}. 
Then, dropping the \m{{}^{a\to b}} superscript for clarity,  
\begin{equation}\label{eq: promote}
\widetilde F_{\sseq{i}{k}}=\chi_{i_1}^{\phantom{a}j_1}\,\chi_{i_2}^{\phantom{a}j_2}\,\ldots\,
\chi_{i_k}^{\phantom{a}j_k}\,F_{\sseq{j}{k}}
\end{equation}
is a \m{k}'th order \m{P}--tensor with respect to permutations of \m{(e_{q_1},\ldots,e_{q_{m'}})}. 
\end{proposition}
 
Equation \ref{eq: promote} tells us that when node \m{\neuron_b} aggregates \m{P}--tensors from its 
children, it first has to ``promote'' them to being \m{P}--tensors with respect to the contents of 
its own receptive field by contracting along each of their dimensions with the appropriate 
\m{\chi^{a\to b}} matrix. This is a critical element in comp-nets to guarantee covariance. 

\begin{proposition}\label{prop: stacking}
Let \m{\neuron_{c_1},\ldots,\neuron_{c_s}} be the children of \m{\neuron_t} in a 
message passing type comp-net with corresponding \m{k}'th order 
tensor activations \m{\smash{F_{c_1},\ldots,F_{c_s}}}. 
Let 
\[[\widetilde F_{c_u}]_{\sseq{i}{k}}=
[\chi^{c_u\to t}]_{i_1}^{\phantom{a}j_1}\,[\chi^{c_u\to t}]_{i_2}^{\phantom{a}j_2}\,\ldots\,
[\chi^{c_u\to t}]_{i_k}^{\phantom{a}j_k}\,[F_{c_u}]_{\sseq{j}{k}}\]
be the promotions of these activations to \m{P}--tensors of \m{\neuron_t}. 
Assume that \m{\prt_t=(e_{p_1},\ldots,e_{p_m})}. 
Now let \m{\smash{\wbar{F}}} be a \m{k\<+1}'th order object in which the \m{j}'th slice is 
\m{\smash{\widetilde F_{p_j}}} 
if \m{\neuron_{p_j}} is one of the children of  \m{\neuron_t}, i.e., 
\[\wbar{F}_{\sseq{i}{k},j}=[\widetilde F_{p_j}]_{\sseq{i}{k}},\]
and zero otherwise. Then \m{\smash{\wbar{F}}} is a \m{k\<+1}'th order \m{P}--tensor of \m{\neuron_t}.  
\end{proposition}

Finally, as already mentioned, the restriction of the adjacency matrix to \m{\prt_i} is a second order 
\m{P}--tensor, which gives an easy way of explicitly adding topological information to the activation. 

\begin{proposition}\label{prop: adjacency}
If \m{F_i} is a \m{k}'th order \m{P}--tensor at node \m{\neuron_i}, and \m{A\tdown_{\prt_i}} is the 
restriction of the adjacency matrix to \m{\prt_i} as defined in Section \ref{sec: second order}, 
then \m{F\<\otimes A\tdown_{\prt_i}} is a \m{k\<+2}'th order \m{P}--tensor.  
\end{proposition} 

\subsection{The general aggregation function and its special cases}

Combining all the above results, assuming that node \m{\neuron_t} has children 
\m{\neuron_{c_1},\ldots,\neuron_{c_s}}, we arrive at the following general 
algorithm for the aggregation rule \m{\smash{\Phi_t}}:
\bigskip 

\begin{center}
\framebox{ 
\parbox{.95\textwidth}{\vspace{3pt}
\setlength{\plitemsep}{2pt}
\begin{compactenum}[~~1.]
\item Collect all the \m{k}'th order activations \m{\smash{F_{c_1},\ldots,F_{c_s}}} of the children. 
\item Promote each activation to \m{\widetilde F_{c_1},\ldots,\widetilde F_{c_s}} (Proposition \ref{prop: promotion}). 
\item Stack \m{\smash{\widetilde F_{c_1},\ldots,\widetilde F_{c_s}}} together into a \m{k\<+1} 
order tensor \m{\smash{T}} (Proposition \ref{prop: stacking}).   
\item Optionally form the tensor product of \m{\smash{T}} with \m{A\tdown_{\prt_t}} to get 
a \m{k\<+3} order tensor \m{H} (otherwise just set \m{H=\smash{T}}) (Proposition \ref{prop: adjacency}). 
\item Contract \m{H} along some number of combinations of dimensions to get \m{s} separate 
lower order tensors \m{Q_{1},\ldots,Q_{s}} (Proposition \ref{prop: contractions}).
\item Mix \m{Q_{1},\ldots,Q_{s}} with a matrix \m{W\tin\RR^{s'\times s}} and apply a nonlinearity 
\m{\Upsilon}
to get the final activation of the neuron, which consists of the \m{s'} output tensors 
\[F^{(i)}=\Upsilon\sqbbigg{\;\sum_{j=1}^{s} W_{i,j}\,Q_j+b_i \mathbbm{1}_{}\,}
\qqquad i=1,2,\ldots s',\]
where the \m{b_i} scalars are bias terms, and \m{\mathbbm{1}} is the 
\m{\abs{\prt_t}\times\ldots\times \abs{\prt_t}} dimensional all ones tensor.
\end{compactenum}
\mbox{}\vspace{0pt}\mbox{}
}
}
\end{center}
\bigskip 

A few remarks are in order about this general scheme:
\begin{compactenum}[~~1.]
\item Since \m{\smash{\widetilde F_{c_1},\ldots,\widetilde F_{c_s}}} are stacked into a larger tensor 
and then possibly also multiplied by \m{A\tdown_{\prt_t}}, the general tendency would be for the 
tensor order to increase at every node, and the corresponding storage requirements to increase exponentially. 
The purpose of the contractions in Step 5 is to counteract this tendency, and pull the order of the tensors 
back to some small number, typically \m{1,2} or \m{3}. 
\item However, since contractions can be done in many different ways, the number of channels will increase. 
When the number of input channels is small, this is reasonable, since otherwise the number of learnable 
weights in the algorithm would be too small. However, if unchecked, this can also become problematic.  
Fortunately, mixing the channels by \m{W} on Step 6 gives an opportunity to stabilize the number of channels 
at some value \m{s'}. 
\item In the pseudocode above, for simplicity, the number of input channels is one and the number of 
output channels is \m{s'}. More realistically, the inputs would also have multiple channels (say, \m{s_0}) 
which would be propagated through the algorithm independently up to the mixing stage, making 
\m{W} an \m{s'\times s\times s_0} dimension tensor (not in the \m{P}--tensor sense!).  
\item The conventional part of the entire algorithm is Step 6, and the only learnable parameters 
are the entries of the \m{W} matrix (tensor) and the \m{b_i} bias terms. 
These parameters are shared by all nodes in the network and learned in the usual way, by stochastic 
gradient descent. 
\item Our scheme could be elaborated further while maintaining permutation covariance by, for example 
taking the tensor product of \m{T} with itself, or by introducing \m{A\tdown_{\prt_t}} 
in a different way. However, the way that \m{\smash{\widetilde F_{c_1},\ldots,\widetilde F_{c_s}}} 
and \m{A\tdown_{\prt_t}} are combined by tensor products is already much more general and expressive 
than conventional message passing networks. 
\item Our framework admits many design choices, including the choice of the order odf the activations, 
the choice of contractions, and \m{c'}. 
However, the overall structure of Steps 1--5 is fully dictated by the covariance 
constraint on the network. 
\item The final output of the network \m{\phi(G)\<=F_r} must be permutation invariant. That means that 
the root node \m{\neuron_r} must produce a tuple of zeroth order tensors (scalars) 
\m{\smash{(F_r^{(1)},\ldots, F_r^{(c)})}}. This is similar to how many other graph representation 
algorithms compute \m{\phi(G)} by summing the activations at level \m{L} or creating histogram features. 
\end{compactenum}
We consider a few special cases to explain how tensor aggregation relates to more conventional 
message passing rules.

\subsubsection{Zeroth order tensor aggregation}

Constraining both the input tensors \m{\smash{F_{c_1},\ldots,F_{c_s}}} and the outputs 
to be zeroth order tensors, i.e., scalars, and foregoing multiplication by \m{A\tdown_{\prt_t}} greatly simplifies 
the form of \m{\Phi}. In this case there is no need for promotions, and \m{T} is just the vector 
\m{(\smash{F^\ell_{c_1},\ldots,F^\ell_{c_s}})}. There is only one way to contract a vector 
into a scalar, and that is to sum its elements. Therefore, in this case, the entire aggregation 
algorithm reduces to the simple formula 
\[F_i=\Upsilon \brBig{\,w \sum_{u=1}^{c} F_{c_u}+b\,}.\]
For a neural network this is too simplistic. However, it's interesting to note that the 
Weisfeiler--Lehmann isomorphism test essentially builds on just this formula, with 
a specific choice of \m{\Upsilon} \citep{Read1977}. 
If we allow more channels in the inputs and the outputs, \m{W} becomes a matrix, and 
we recover the simplest form of neural message passing algorithms \citep{Duvenaud2015}.

\subsubsection{First order tensor aggregation}

In first order tensor aggregation, assuming that \m{\abs{\prt_i}\<=m}, 
\m{\smash{\widetilde F_{c_1},\ldots,\widetilde F_{c_s}}} are \m{m} dimensional column vectors, and 
\m{T} is an \m{m\times m} matrix consisting of \m{\smash{\widetilde F_{c_1},\ldots,\widetilde F_{c_s}}} 
stacked columnwise. 
There are two ways of contracting (in our generalized sense) a matrix into a vector: 
by summing over its rows, or summing over its columns. The second of these choices leads us back to 
summing over all contributions from the children, while the first is more interesting because it 
corresponds to summing \m{\smash{\widetilde F_{c_1},\ldots,\widetilde F_{c_s}}} as vectors individually. 
In summary, we get an aggregation function that transforms a single input channel to two output 
channels of the form 
\[
F_i^{(1)}=\Upsilon\sqbBig{\: w_{1,1} (T^\top\nts \V 1) + w_{1,2} (T\,\V {1}) + b_1\ts \V 1\:},
\qquad
F_i^{(2)}=\Upsilon\sqbBig{\: w_{2,1} (T^\top\nts \V 1) + w_{2,2} (T\,\V {1}) + b_2\ts \V 1\:},
\]
where \m{\V 1} denotes the \m{m} dimensional all ones vector. 
Thus, in this layer \m{W\tin\RR^{2\times 2}}. 
Unless constrained by \m{c'}, in each subsequent layer the number of channels doubles 
further and these channels can all mix with each other, 
so \m{W^{(2)}\nts \tin \RR^{4\times 4}}, \m{W^{(3)}\nts \tin \RR^{8\times 8}}, and so on.

\subsubsection{Second order tensor aggregation without the adjacency matrix} 

In second order tensor aggregation, \m{T} is a third order \m{P}--tensor, which can be contracted 
back to second order in three different ways, by projecting it along each of its dimensions. 
Therefore the outputs will be the three matrices 
\[F^{(i)}=\Upsilon\brbig{
w_{i,1} T\tdown_{1}+w_{i,2} T\tdown_{2}+w_{i,3} T\tdown_{3}+b_i\V 1_{m\times m}
}\qqquad\qqquad i\tin\cbrN{1,2,3},\]
and the weight matrix is \m{W\tin\RR^{3\times 3}}. 

\subsubsection{Second order tensor aggregation with the adjacency matrix}

The first nontrivial tensor contraction case occurs when 
\m{\smash{\widetilde F_{c_1},\ldots,\widetilde F_{c_s}}} are second order tensors, and 
we multiply with \m{A\tdown_{\prt_t}}, since in that case \m{T} is 5th order, and can be contracted 
down to second order in a total of 50 different ways:
\begin{compactenum}[~~1.]
\item The ``1+1+1'' case contracts \m{T} in the form \m{T_{i_1,i_2,i_3,i_4,i_5}
\delta^{i_{a_1}}\delta^{i_{a_2}}\delta^{i_{a_3}}}, i.e., it projects \m{T} down along \m{3} of its \m{5} 
dimensions. This alone can be done in \m{\smash{\binom{5}{3}\<=10}} different ways\footnote{For simplicity, 
we ignore the fact that symmetries, such as the symmetry of \m{A\tdown_{\prt_t}}, might reduce the 
number of distinct projections somewhat.}
\item The ``1+2'' case contracts \m{T} in the form \m{T_{i_1,i_2,i_3,i_4,i_5}
\delta^{i_{a_1}}\delta^{i_{a_2},i_{a_3}}}, i.e., it projects \m{T} along one dimension,  
and contracts it along two others. This can be done in \m{\smash{3\binom{5}{3}\<=30}} ways. 
\item The ``3'' case is a single \m{3}-fold contraction 
\m{T_{i_1,i_2,i_3,i_4,i_5} \delta^{i_{a_1},i_{a_2},i_{a_3}}}, which again can be done 
in \m{\smash{\binom{5}{3}\<=10}} different ways. 
\end{compactenum}
Clearly, maintaining 50 channels in a message passing architecture is excessive, so 
in practice it is reasonable to set \m{c'\approx 10}, making \m{W\tin\RR^{10\times 50}}. 
\input{fig-TensorNeigh}
\ignore{\input{fig-2D-tensor}}
\ignore{
The general formula for the number of distinct contractions of a rank \m{k} \m{P}--tensor 
to a rank \m{k-p} \m{P}--tensor is given by the following formula.

\begin{proposition}
Let \m{T} be a rank \m{k} tensor. Then the number of distinct generalized contractions 
of \m{T} down to  rank \m{k-p} is ... 
\end{proposition}
\begin{proof}
Each contraction must involve exactly \m{p} vertices, which we shall call \m{\sseq{a}{p}}. 
The set \m{\cbrN{\sseq{a}{p}}} can be chosen in \m{\binom{k}{p}} distinct ways. 
This set must then be partitioned into a union of subsets corresponding to variables 
involved in the \emph{same} contraction (featured as indices of the same \m{\delta}). 
Since there are

\end{proof}
}

%% file: fig-TensorNeigh.tex
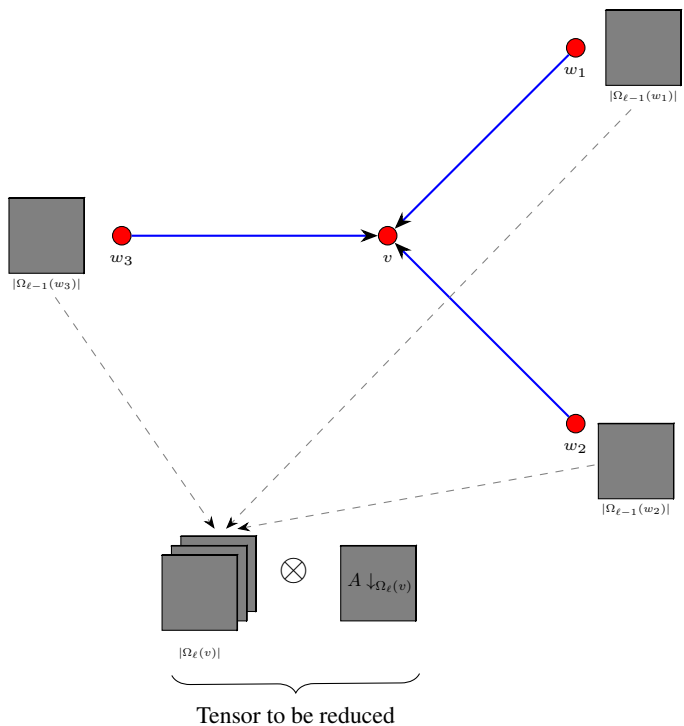
\begin{figure}[t]
	\begin{center}
\begin{tikzpicture}[scale = 0.5,decoration={brace,mirror,amplitude=7},>={Stealth[black]}]

  \node[circle,draw=black, fill=red, inner sep=0pt,minimum size=7pt,label={below: \tiny $v$}] (root1) at (0,0) {};
  
  \node[circle,draw=black, fill=red, inner sep=0pt,minimum size=7pt,label={below: \tiny $w_1$}] (vertex1) at (5,5) {};
  	\pic [fill=gray, opacity=1] at (7.8,6) {annotated cuboid={width=100, height=100, depth=1, scale=.01, units=m}};
  	\node[scale=0.7] (TensorDim1) at (6.8,3.7) {\tiny $|\Omega_{\ell-1}(w_1)|$};
  
  \node[circle,draw=black, fill=red, inner sep=0pt,minimum size=7pt,label={below: \tiny $w_2$}] (vertex2) at (5,-5) {};
    \pic [fill=gray, opacity=1] at (7.6,-5) {annotated cuboid={width=100, height=100, depth=1, scale=.01, units=m}};
    \node[scale=0.7] (TensorDim2) at (6.6,-7.3) {\tiny $|\Omega_{\ell-1}(w_2)|$};
  
  \node[circle,draw=black, fill=red, inner sep=0pt,minimum size=7pt,label={below: \tiny $w_3$}] (vertex3) at (-7.07,0) {};
    	\pic [fill=gray, opacity=1] at (-8.07,1) {annotated cuboid={width=100, height=100, depth=1, scale=.01, units=m}};
    	\node[scale=0.7] (TensorDim3) at (-9.07,-1.3) {\tiny $|\Omega_{\ell-1}(w_3)|$};
  
  \draw[thick,->, blue] (vertex1) -- (root1); \draw[thick,->, blue] (vertex2) -- (root1);
  \draw[thick,->, blue] (vertex3) -- (root1);
  
   \pic [fill=gray, opacity=1] at (-3.5,-8) {annotated cuboid={width=100, height=100, depth=1, scale=.01, units=m}};
   \node[scale=0.7] (opTensorDim3) at (-5,-11.1) {\tiny $|\Omega_{\ell}(v)|$};
   \pic [fill=gray, opacity=1] at (-3.75,-8.25) {annotated cuboid={width=100, height=100, depth=1, scale=.01, units=m}};
   \pic [fill=gray, opacity=1] at (-4.0,-8.5) {annotated cuboid={width=100, height=100, depth=1, scale=.01, units=m}};
   \node[scale=1.0] (kp) at (-2.5,-8.9) {\Large $\otimes$};
   
   \pic [fill=gray, opacity=1] at (0.75,-8.25) {annotated cuboid={width=100, height=100, depth=1, scale=.01, units=m}};
   
   \node[scale=0.7] (redadjmatrix) at (-0.2,-9.25) {$A\downarrow_{\Omega_{\ell}(v)}$};
      
   \draw [decorate] ([yshift=-6mm]opTensorDim3.west) --node[below=3mm]{\small Tensor to be reduced} ([yshift=-24mm]redadjmatrix.east);

   \draw[thin,dashed, ->, gray] (TensorDim3) -- (-4.6,-7.8);
   \draw[thin,dashed, ->, gray] (TensorDim1) -- (-4.3,-7.8);
   \draw[thin,dashed, ->, gray] (5.5,-6.1) --(-4.0,-7.8);   

\end{tikzpicture}
\end{center}
\caption{\label{fig: Tensor Neighbors}
The activations of each neighbor are stacked into a tensor \m{T} which is tensor multiplied by the 
restriction of the adjacency matrix, and then reduced in different ways.
}
\end{figure}

%% file: experiments.tex
\section{Experiments}

\ignore{
To test the effectiveness of our proposed framework for constructing
covariant compositional neural networks (CCNs),
we implemented and tested a second order CCN, as described
in section 5, 
on the Harvard Clean Energy Project dataset \citep{HCEP}, 
and standard graph kernel datasets: MUTAG \citep{MUTAG}, PTC \citep{PTC}, NCI1, and
NCI109 \citep{NCI}.
}

We compared the second order variant of our 
CCN framework (Section \ref{sec: second order}) 
to several standard graph learning algorithms on three types of datasets that 
involve learning the properties of molecules from their structure: 
\begin{enumerate}[~1.]
\item 
\textbf{The Harvard Clean Energy Project} \citep{HCEP}, consisting of 2.3 million 
organic compounds that are candidates for use in solar cells. 
The regression target in this case is Power Conversion Efficiency (PCE). 
Due to time constraints, instead of using the entire dataset, the experiments were ran on 
a random subset of 50,000 molecules. 
\item 
\textbf{QM9}, which is a dataset of all ~133k organic molecules  
with up to nine atoms (C,H,O,N and F) 
out of the GDB-17 universe of molecules. Each molecule has
13 target properties to predict. 
The dataset does contain spatial information relating to the atomic configurations, 
but we only used the chemical graph and atom node labels. 
For our experiments we normalized each target variable to have mean 0
and standard deviation 1.
\item 
\textbf{Graph kernels datasets}, specifically 
\begin{inparaenum}
\item 
MUTAG, which is a dataset of 188 mutagenic aromatic and heteroaromatic compounds \citep{MUTAG}; 
\item 
PTC, which consists of 344
chemical compounds that have been tested for positive or negative toxicity in lab rats \citep{PTC};  
\item 
NCI1 and NCI109, which have 4110 and 4127 compounds respectively, 
each screened for activity against small cell lung cancer and ovarian cancer lines \citep{NCI}. 
\end{inparaenum}
\end{enumerate}

In the case of HCEP, we compared CCN to 
lasso, ridge regression, random forests, gradient
boosted trees, optimal assignment Wesifeiler--Lehman graph kernel \citep{Kriege16} (WL), 
neural graph fingerprints \citep{Duvenaud2015},
and the ``patchy-SAN'' convolutional type algorithm from \citep{Niepert2016} (referred to as PSCN).
For the first four of these baseline methods, we created simple feature vectors from each
molecule: the number of bonds of each type (i.e. number of H--H bonds, number of C--O bonds,
etc) and the number of atoms of each type. Molecular
graph fingerprints uses atom labels of each vertex as base features.
For ridge regression and lasso, we cross validated over \m{\lambda}. For random forests and gradient boosted
trees, we used 400 trees, and cross validated over max depth, minimum samples for a leaf, minimum
samples to split a node, and learning rate (for GBT).
For neural graph fingerprints, we used 2 layers and a hidden layer size of 10. In PSCN, we used a
patch size of 10 with two convolutional layers and a dense layer on top as described in their
paper.

For the graph kernels datasets, we compare against graph kernel results as reported in 
\citep{KondorPan2016} (which computed kernel matrices using the Weisfeiler--Lehman, Weisfeiler--edge,
shortest paths, graphlets and multiscale Laplacian graph kernels and used a C-SVM on top), Neural
graph fingerprints (with 2 levels and a hidden size of 10) and PSCN. 
For QM9, we compared against the Weisfeiler--Lehman graph kernel (with C-SVM on top),
neural graph fingerprints, and PSCN. The settings for NGF and PSCN
are as described for HCEP. 

For our own method, second order CCN, 
we initialized the base features of each vertex with computed histogram alignment kernel
features \citep{Kriege16} of depth up to 10. 
Each vertex receives a base label $l_i= \text{concat}_{j=1}^{10} {H_j (i)}$ where
$H_j(i) \in \mathbb{R}^d$  (with $d$ being the total number of distinct discrete node labels) is the vector 
of relative frequencies of each label for the set of vertices at distance equal to $j$ from vertex $i$.
We used two levels and doubled the intermediate
channel size at each level. 
For computational efficiency,  we only used 10 contractions as described in section 5.1.4 instead
of the full 50 contractions.  

In each experiment we used 80\% of the dataset for training, 10\% for validation, and
evaluated on the remaining 10\% test set. For the kernel datasets we performed
the experiments on 10 separate training/validation/test stratified splits and averaged the resulting
classification accuracies.
We always used stochastic gradient descent with momentum
0.9. Our initial learning rate was set to 0.001 after experimenting on a held out set. The learning
rate decayed linearly after each step towards a minimum of \m{10^{-6}}.

\subsection{Discussion}
On the subsampled HCEP dataset, CCN outperforms all other methods 
by a very large margin.
For the graph kernels datasets, SVM with the Weisfeiler--Lehman kernels achieve the highest
accuracy on NCI1 and NCI109, while CCN wins on MUTAG and PTC. 
Perhaps this poor performance is to be expected, since the datasets are small and
neural network approaches usually require tens of thousands of training examples at minimum
to be effective. Indeed, neural graph fingerprints and PSCN also perform poorly compared to
the Weisfeiler--Lehman kernels.

In the QM9 experiments, CCN beats the three other algorithms in both mean absolute error and 
root mean squared error. 
It should be noted that \citep{Gilmer2017} obtained stronger results on QM9, but
we cannot properly compare our results with theirs because our experiments
only use the adjacency matrices and atom labels of each node, while theirs
includes comprehensive chemical features that better inform the target quantum
properties. 

\input{tbl_results}\label{tbl: results}


%% file: tbl_results.tex
\ignore{
\begin{table}[t]
\caption{\label{tbl: hcep-old} HCEP regression results}
\begin{tabular}{||l | c | c | c | c ||}
	\hline
	& Train MAE & Train RMSE & Test MAE & Test RMSE \\
	\hline\hline
	Lasso & 0.863 & 1.190 & 0.867 & 1.437 \\
	\hline
	Ridge regression & 0.849 & 1.164 & 0.854 & 1.376 \\
	\hline
	Random forest & 0.999 & 1.331 & 1.004 & 1.799 \\
	\hline
	Gradient boosted trees & 0.676 & 0.939 & 0.704 & 1.005 \\
	\hline
	Weisfeiler--Lehman graph kernel & 0.805 & 1.111 & 0.805 & 1.096 \\
	\hline
	Neural graph fingerprints & 0.848 & 1.187 & 0.851 & 1.177 \\
	\hline
	PSCN ($k\< = 10$) & 0.704 & 0.972 & 0.718 & 0.973 \\ 
	\hline
	CCN 2D & {\color{red} 0.320} & {\color{red} 0.422} & {\color{red} 0.340} & {\color{red} 0.449} \\
	\hline
\end{tabular}
\end{table}
}

\begin{table}[t]
\caption{\label{tbl: hcep} HCEP regression results}
\begin{tabular}{||l | c | c ||}
	\hline
	& Test MAE & Test RMSE \\
	\hline\hline
	Lasso & 0.867 & 1.437 \\
	\hline
	Ridge regression & 0.854 & 1.376 \\
	\hline
	Random forest & 1.004 & 1.799 \\
	\hline
	Gradient boosted trees & 0.704 & 1.005 \\
	\hline
	WL graph kernel & 0.805 & 1.096 \\
	\hline
	Neural graph fingerprints & 0.851 & 1.177 \\
	\hline
	PSCN $(k \<= 10)$ & 0.718 & 0.973 \\ 
	\hline
	Second order CCN (our method) & {\color{red} 0.340} & {\color{red} 0.449} \\
	\hline
\end{tabular}
\end{table}

\begin{table}[t]
\caption{\label{tbl: kernels} Kernel Datasets Classification results (accuracy +/- standard deviation)}
\begin{tabular}{||l | c | c | c | c ||}
	\hline
	& MUTAG & PTC & NCI1 & NCI109 \\
	\hline\hline
	WL & 84.50 $\pm$ 2.16 & 59.97 $\pm$ 1.60 & {\color{red} 84.76 $\pm$ 0.32} & {\color{red} 85.12 $\pm$ 0.29} \\
	\hline
	WL-edge & 82.94 $\pm$ 2.33 & 60.18 $\pm$ 2.19 & {\color{red} 84.65 $\pm$ 0.25} & {\color{red} 85.32 $\pm$ 0.34} \\
	\hline
	SP & 85.50 $\pm$ 2.50 & 59.53 $\pm$ 1.71 & 73.61 $\pm$ 0.36 & 73.23 $\pm$ 0.26 \\
	\hline
	Graphlet & 82.44 $\pm$ 1.29 & 55.88 $\pm$ 0.31 & 62.40 $\pm$ 0.27 & 62.35 $\pm$ 0.28 \\
	\hline
	p-RW & 80.33 $\pm$ 1.35 & 59.85 $\pm$ 0.95 & TIMED OUT & TIMED OUT \\
	\hline
	MLG & 87.94 $\pm$ 1.61 & 63.26 $\pm$ 1.48 & 81.75 $\pm$ 0.24 & 81.31 $\pm$ 0.22 \\
	\hline
	PSCN $k = 10$ (Niepert et al.) & 88.95 $\pm$ 4.37 & 62.29 $\pm$ 5.68 & 76.34 $\pm$ 1.68 & N/A \\
	\hline
	Neural graph fingerprints & 89.00 $\pm$ 7.00 & 57.85 $\pm$ 3.36 & 62.21 $\pm$ 4.72 & 56.11 $\pm$ 4.31 \\
	\hline
	Second order CCN (our method) & {\color{red} 91.64 $\pm$ 7.24} & {\color{red} 70.62 $\pm$ 7.04} & 76.27 $\pm$ 4.13 & 75.54 $\pm$ 3.36 \\
	\hline
\end{tabular}
\end{table}

\begin{table}
\caption{\label{tbl: qm9mae} QM9 regression results (MAE)}
\begin{center}
\begin{tabular}{|| l | c | c | c | c ||}
	\hline
	                    & WLGK 		& NGF 	& PSCN ($k\<= 10$)  &  Second order CCN   \\
	\hline\hline
	alpha               & 0.46		& 0.43	& 0.20  & {\color{red} 0.16}	    \\
	\hline
	Cv	                & 0.59  	& 0.47	& 0.27  & {\color{red} 0.23}	    \\
	\hline
	G 	                & 0.51		& 0.46	& 0.33  & {\color{red} 0.29}	    \\
	\hline
	gap 	            & 0.72		& 0.67	& 0.60  & {\color{red} 0.54}	    \\
	\hline
	H 	                & 0.52		& 0.47	& 0.34  & {\color{red} 0.30}	    \\
	\hline
	HOMO 	            & 0.64		& 0.58	& 0.51  & {\color{red} 0.39}      \\
	\hline
	LUMO 	            & 0.70		& 0.65	& 0.59  & {\color{red} 0.53}	    \\
	\hline
	mu	 				& 0.69		& 0.63	& 0.54  & {\color{red} 0.48}		\\
	\hline
	omega1 				& 0.72		& 0.63	& 0.57  & {\color{red} 0.45}		\\
	\hline
	R2	 				& 0.55		& 0.49	& 0.22  & {\color{red} 0.19}	    \\
	\hline
	U	 	            & 0.52		& 0.47	& 0.34  & {\color{red} 0.29}      \\
	\hline
	U0	                & 0.52		& 0.47	& 0.34  & {\color{red} 0.29}	    \\
	\hline
	ZPVE 	            & 0.57		& 0.51	& 0.43  & {\color{red} 0.39}	    \\
	\hline
\end{tabular}
\end{center}
\end{table}

\ignore{
\begin{table}
\caption{\label{tbl: qm9mae-ccn1d} QM9 regression results (MAE)}
\begin{center}
\begin{tabular}{|| l | c | c | c | c | c ||}
	\hline
	                    & WLGK 		& NGF 	& PSCN ($k \<= 10$)  & CCN 1D 		&  Second order CCN   \\
	\hline\hline
	alpha               & 0.46		& 0.43	& 0.20  & {\color{red} 0.18}	& {\color{red} 0.16}	    \\
	\hline
	Cv	                & 0.59  	& 0.47	& 0.27  & {\color{red} 0.23}	& {\color{red} 0.23}	    \\
	\hline
	G 	                & 0.51		& 0.46	& 0.33  & {\color{red} 0.29}	& {\color{red} 0.29}	    \\
	\hline
	gap 	            & 0.72		& 0.67	& 0.60  & {\color{red} 0.56}	& {\color{red} 0.54}	    \\
	\hline
	H 	                & 0.52		& 0.47	& 0.34  & {\color{red} 0.29}	& {\color{red} 0.30}	    \\
	\hline
	HOMO 	            & 0.64		& 0.58	& 0.51  & {\color{red} 0.40}	& {\color{red} 0.39}      \\
	\hline
	LUMO 	            & 0.70		& 0.65	& 0.59  & {\color{red} 0.54}	& {\color{red} 0.53}	    \\
	\hline
	mu	 				& 0.69		& 0.63	& 0.54  & {\color{red} 0.49}	& {\color{red} 0.48}		\\
	\hline
	omega1 				& 0.72		& 0.63	& 0.57  & {\color{red} 0.48}	& {\color{red} 0.45}		\\
	\hline
	R2	 				& 0.55		& 0.49	& 0.22  & {\color{red} 0.19}	& {\color{red} 0.19}	    \\
	\hline
	U	 	            & 0.52		& 0.47	& 0.34  & {\color{red} 0.30} 	& {\color{red} 0.29}      \\
	\hline
	U0	                & 0.52		& 0.47	& 0.34  & {\color{red} 0.29}	& {\color{red} 0.29}	    \\
	\hline
	ZPVE 	            & 0.57		& 0.51	& 0.43  & {\color{red} 0.38}	& {\color{red} 0.39}	    \\
	\hline
\end{tabular}
\end{center}
\end{table}
}

\begin{table}
\caption{\label{tbl: qm9rmse} QM9 regression results (RMSE)}
\begin{center}
\begin{tabular}{|| l | c | c | c | c ||}
	\hline
	                    & WLGK 		& NGF 	& PSCN ($k \<= 10$) 	&  Second order CCN   \\
	\hline\hline
	alpha               & 0.68		& 0.65	& 0.31  & {\color{red} 0.26}	    \\
	\hline
	Cv	                & 0.78  	& 0.65	& 0.34  & {\color{red} 0.30}	    \\
	\hline
	G 	                & 0.67		& 0.62	& 0.43  & {\color{red} 0.38}	    \\
	\hline
	gap 	            & 0.86		& 0.82	& 0.75  & {\color{red} 0.69}	    \\
	\hline
	H 	                & 0.68		& 0.62	& 0.44  & {\color{red} 0.40}	    \\
	\hline
	HOMO 	            & 0.91		& 0.81	& 0.70  & {\color{red} 0.55}      \\
	\hline
	LUMO 	            & 0.84		& 0.79	& 0.73  & {\color{red} 0.68}	    \\
	\hline
	mu	 				& 0.92		& 0.87	& 0.75  & {\color{red} 0.67}		\\
	\hline
	omega1 				& 0.84		& 0.77	& 0.73  & {\color{red} 0.65}		\\
	\hline
	R2	 				& 0.81		& 0.71	& 0.31  & {\color{red} 0.27}	    \\
	\hline
	U	 	            & 0.67		& 0.62	& 0.44  & {\color{red} 0.40}      \\
	\hline
	U0	                & 0.67		& 0.62	& 0.44  & {\color{red} 0.39}	    \\
	\hline
	ZPVE 	            & 0.72		& 0.66	& 0.55  & {\color{red} 0.51}	    \\
	\hline
\end{tabular}
\end{center}
\end{table}

\ignore{
\begin{table}
\caption{\label{tbl: qm9rmse-ccn1d} QM9 regression results (RMSE)}
\begin{center}
\begin{tabular}{|| l | c | c | c | c | c ||}
	\hline
	                    & WLGK 		& NGF 	& PSCN ($k \<= 10$) 	& CCN 1D 				&  CCN 2D   \\
	\hline\hline
	alpha               & 0.68		& 0.65	& 0.31  & {\color{red} 0.29}	& {\color{red} 0.26}	    \\
	\hline
	Cv	                & 0.78  	& 0.65	& 0.34  & {\color{red} 0.31}	& {\color{red} 0.30}	    \\
	\hline
	G 	                & 0.67		& 0.62	& 0.43  & {\color{red} 0.39}	& {\color{red} 0.38}	    \\
	\hline
	gap 	            & 0.86		& 0.82	& 0.75  & {\color{red} 0.71}	& {\color{red} 0.69}	    \\
	\hline
	H 	                & 0.68		& 0.62	& 0.44  & {\color{red} 0.40}	& {\color{red} 0.40}	    \\
	\hline
	HOMO 	            & 0.91		& 0.81	& 0.70  & {\color{red} 0.56}	& {\color{red} 0.55}      \\
	\hline
	LUMO 	            & 0.84		& 0.79	& 0.73  & {\color{red} 0.68}	& {\color{red} 0.68}	    \\
	\hline
	mu	 				& 0.92		& 0.87	& 0.75  & {\color{red} 0.67}	& {\color{red} 0.67}		\\
	\hline
	omega1 				& 0.84		& 0.77	& 0.73  & {\color{red} 0.66}	& {\color{red} 0.65}		\\
	\hline
	R2	 				& 0.81		& 0.71	& 0.31  & {\color{red} 0.28}	& {\color{red} 0.27}	    \\
	\hline
	U	 	            & 0.67		& 0.62	& 0.44  & {\color{red} 0.40}	& {\color{red} 0.40}      \\
	\hline
	U0	                & 0.67		& 0.62	& 0.44  & {\color{red} 0.40}	& {\color{red} 0.39}	    \\
	\hline
	ZPVE 	            & 0.72		& 0.66	& 0.55  & {\color{red} 0.51}	& {\color{red} 0.51}	    \\
	\hline
\end{tabular}
\end{center}
\end{table}
}

%% file: conclusion.tex
\section{Conclusions}

We have presented a general framework called covariant compositional networks (CCNs)
for constructing covariant graph neural networks, which
encompasses other message passing approaches as special cases, but takes a more general and principled 
approach to ensuring covariance with respect to permutations.  
Experimental results on several benchmark datasets show that CCNs can outperform other state-of-the-art algorithms.

%% file: app-background.tex
\section{Mathematical Background}

\paragraph{Groups.} A \df{group} is a set \m{G} endowed with an operation \m{G\times G\to G} (usually denoted multiplicatively) obeying 
the following axioms: 
\begin{compactenum}[~G1.]
	\item for any \m{u,v\tin G},~ \m{uv\tin G} (closure);
	\item for any \m{u,v,w\tin G},~ \m{u(vw)=(uv)w} (associativity);
	\item there is a unique \m{e\tin G}, called the \df{identity} of \m{G}, such that \m{eu=ue=u} for any \m{u\tin G};
	\item for any \m{u\tin G}, there is a corresponding element \m{u^{-1}\!\tin G} called the \df{inverse} of \m{u}, such that 
	\m{u\ts u^{-1}=u^{-1}u=e}.
\end{compactenum}
We do \emph{not} require that the group operation be commutative, i.e., in general, \m{uv\neq vu}. 
Groups can be finite or infinite, countable or uncountable, compact or non-compact. 
While most of the results in this paper would generalize to any compact group, the keep the exposition 
as simple as possible,  
throughout we assume that \m{G} is finite or countably infinite. 
As usual, \m{\abs{G}} will denote the size (cardinality) of \m{G}, sometimes also called the \df{order} of the group. 

\paragraph{Representations.} 
A (finite dimensional) \df{representation} of a group \m{G} over a field \m{\FF} 
is a matrix-valued function \m{R\colon G\to\FF^{d_\rho\times d_\rho}} 
such that \m{R(x)\ts R(y)=R(xy)} for any \m{x,y\tin G}. 
We generally assume that \m{\FF\<=\CC}, however in the special case 
when \m{G} is the symmetric group \m{\Sn} we can restrict ourselves to only considering 
real-valued representations, i.e., \m{\FF=\RR}. 

\ignore{
Two representations \m{\rho_1} and \m{\rho_2} of the same dimensionality \m{d}  
are said to be \df{equivalent} if for some invertible \m{T\tin\CC^{d\times d}},\; \m{\rho_1(x)=T^{-1}\nts\rho_2(x)\,T} for any \m{x\tin G}. 
A representation \m{\rho} is said to be \df{reducible} if  
it decomposes into a direct sum of smaller representations in the form  
\begin{equation*}
\rho(x)=
T^{-1}\br{\rho_1(x)\<\oplus\rho_2(x)}\, T=
T^{-1}\nts \br{
	\begin{array}{c|c}
	\rho_1(x)&0\\
	\hline
	0&\rho_2(x)
	\end{array}
}T\qqquad \forall\, x\tin G
\end{equation*}
for some invertible \m{T\tin\CC^{d_\rho\times d_\rho}}. 
A key theorem in theory of group representations is that any representation of a finite (more generally, compact) 
group over \m{\CC} reduces to a direct sum of irreducible representations. 
\vspace{-2pt}

A maximal system of pairwise inequivalent, irreducible representations we call a system of \df{irreps}. 
It is possible to show that if \m{\Rcal_1} and \m{\Rcal_2} are two different systems of irreps of the same 
finite group \m{G}, then \m{\Rcal_1} and \m{\Rcal_2} have the same (finite) number of irreps and 
there is a bijection \m{\phi\colon\Rcal_1\to\Rcal_2} such that if \m{\phi(\rho_1)=\rho_2}, then \m{\rho_1} and 
\m{\rho_2} are equivalent. 
A finite group always has at least one system of irreps in which all the \m{\rho(x)} matrices are \df{unitary}, i.e., 
\m{\rho(x)^{-1}=\rho(x)^\dag={\rho(x)^\ast}{}^\top} for all \m{x\in G}. 
}

%% file: appendix_proofs.tex
\section{Proofs}

\begin{pfprop}{prop: invariant reps}
Let \m{\Gcal} and \m{\Gcal'} be two compound objects, where \m{\Gcal'} is equivalent to \m{\Gcal} 
up to a permutation \m{\sigma \tin \Sn} of the atoms. 
For any node \m{\neuron_a} of \m{\Gcal} we let \m{\neuron'_a} be the corresponding node of \m{\Gcal'}, 
and let \m{f_a} and \m{f'_a} be their activations. 
   

We prove that \m{f_a\<=f_a'} for every node in \m{\Gcal} by using induction on the distance 
of \m{\neuron_a} from its farthest descendant that is a leaf, which we call its \emph{height} 
and denote \m{h(a)}. 
For \m{h(a)\<=0}, the statment is clearly true, since \m{f_a=f'_a=\ell_{\xi(a)}}. 
Now assume that it is true for all nodes with height up to \m{h^\ast}. 
For any node \m{\neuron_a} with \m{h(a)\<=h^\ast\<+1},\: \m{f_a\<=\Phi(f_{c_1},f_{c_2},\ldots, f_{c_k})}, 
where each of the children \m{\sseq{c}{k}} are of height at most \m{h^\ast}, therefore 
\[f_a=\Phi(f_{c_1},f_{c_2},\ldots, f_{c_k})=\Phi(f'_{c_1},f'_{c_2},\ldots, f'_{c_k})=f'_a.\]
Thus, \m{f_a\<=f_a'} for every node in \m{\Gcal}. 
The proposition follows by \m{\phi(\Gcal) \<= f_r \<= f'_{r} \<= \phi(\Gcal')}. 
\end{pfprop}

\ignore{
\begin{pfprop}[2]
\end{pfprop}
}

\begin{pfprop}{prop: steerability}
Let \m{\Gcal}, \m{\Gcal'}, \m{\Ncal} and \m{\Ncal'} be as in Definition \ref{def: invariant compnet}. 
As in Definition \ref{def: steerable compnet}, for each node (neuron) \m{\neuron_i} in \m{\Ncal} there is 
a node \m{\neuron'_j} in \m{\Ncal'} such that their receptive fields are equivalent up to permutation. 
That is, if \m{\absN{\prt_i}\<=m}, then \m{\absN{\prt'_j}\<=m}, and there is a permutation \m{\pi\tin\Sbb_m}, 
such that if \m{\prt_i\<=\brN{e_{p_1},\ldots, e_{p_m}}} and \m{\prt'_j\<=\brN{e_{q_1},\ldots, e_{q_m}}}, then 
\m{e_{q_{\pi(a)}}\!\<=e_{p_a}}. By covariance, then \m{f'_j\<=R_\pi(f_i)}. 

Now let \m{\Gcal''} be a third equivalent object, and \m{\Ncal''} the corresponding comp-net. 
\m{\Ncal''} must also have a node, \m{\neuron''_k}, that corresponds to \m{\neuron_i} and \m{\neuron'_j}. 
In particular, letting its receptive field be \m{\Pcal''_k=\brN{e_{r_1},\ldots,e_{r_m}}}, 
there is a permutation \m{\sigma\tin\Sbb_m} 
for which \m{e_{r_{\sigma(b)}}\!\<=e_{q_b}}. Therefore, \m{f''_k\<=R_\sigma(f'_j)}. 

At the same time, \m{\neuron''_k} is also in correspondence with \m{\neuron_i}. In particular, letting 
\m{\tau\<=\sigma\pi} (which corresponds to first applying the permutation \m{\pi}, then applying \m{\sigma}), 
\m{e_{r_{\tau(a)}}\!\<=e_{p_a}}, and therefore \m{f''_k\<=R_\tau(f_i)}.  
Hence, the \m{\cbrN{R_\pi}} maps must satisfy 
\[
R_{\sigma\pi}(f_i)=R_\sigma(f_j')=R_\sigma(R_\pi(f_i)),\] 
for any \m{f_i}. More succinctly, \m{R_{\sigma\pi}=R_\sigma\<\circ R_\pi} for any \m{\pi,\sigma\tin\Sbb_m}. 
In the case that the \m{\cbrN{R_\pi}} maps are linear and represented by matrices, 
this reduces to \m{R_{\sigma\pi}=R_\sigma R_\pi}, which is equivalent to saying that they 
form a group representation of \m{\Sbb_m}. 
\ignore{
Let \m{\pi, \sigma, \tau \in \Sn}, with \m{\pi = \sigma \tau}. Let \m{\Gcal} be a compound object
and \m{\Gcal'} be the equivalent object where the atoms have been permuted by \m{\tau}. 
Let \m{\pi = \sigma \circ \tau} and \m{\Gcal'} be the compound object obtained after permuting
the atoms of \m{\Gcal} by \m{\tau}, \m{\Gcal''} be the compound object obtained after permuting
the atoms of \m{\Gcal'} by \m{\sigma}. Let \m{i, j, k \in \{1, ..., n\}}, where
\m{\pi(i) = \sigma \circ \tau(i) = \sigma (j) = k}.
Then \m{f_k^{''} = R_{\sigma}(f_j^{'}) = R_{\sigma} R_{\tau}(f_i) = R_{\sigma} R_{\tau} f_i}.
We also have: \m{f_k^{''} = R_{\sigma \tau} f_i}. Thus \m{R_{\sigma \tau} = R_{\sigma} R_{\tau}} as
desired.}
\end{pfprop}

\begin{pfprop}{prop: contractions}
Under the action of a permutation \m{\pi\tin\Sbb_m}, \m{A} and \m{B} transform as 
\begin{eqnarray}
A\mapsto A' \qqquad \qquad [A']_{\sseq{j}{k}}&=&
[P_\pi]_{j_1}{\vphantom{I}}^{\!\!j'_1}\, [P_\pi]_{j_2}{\vphantom{I}}^{\!\!j'_2}\:  \ldots\: 
[P_\pi]_{j_k}{\vphantom{I}}^{\!\!j'_k}\:[A]_{\sseq{j'}{k}}, \label{eq: act1} \\  
B\mapsto B' \qqquad \qquad [B']_{\sseq{j}{p}}&=&
[P_\pi]_{j_1}{\vphantom{I}}^{\!\!j'_1}\, [P_\pi]_{j_2}{\vphantom{I}}^{\!\!j'_2}\:  \ldots\: 
[P_\pi]_{j_p}{\vphantom{I}}^{\!\!j'_p}\:[B]_{\sseq{j'}{p}}. \label{eq: act2}  
\end{eqnarray}

\par{\textbf{Case 1.}~}
Let $C=A\<\otimes B$. Under \rf{eq: act1} and \rf{eq: act2}, \m{C} transforms into 
\begin{multline*}
[C']_{\sseq{i}{k+p}} = 
\br{[P_\pi]_{i_1}{\vphantom{I}}^{\!\!i'_1}\, \ldots [P_\pi]_{i_k}^{{\vphantom{I}}^{\!\!i'_k}} \;[A]_{\sseq{i'}{k}}} 
\br{[P_\pi]_{i_{k+1}}{\vphantom{I}}^{\!\!\!\!i'_{k+1}}   \ldots 
[P_\pi]_{i_{k+p}}{\vphantom{I}}^{\!\!\!\!i'_{k+p}} \;[B]_{i'_{k+1}, ... ,i'_{k+p}}} \\
= [P_\pi]_{i_1}{\vphantom{I}}^{\!\!i'_1} \: \ldots \: [P_\pi]_{i_{k+p}}{\vphantom{I}}^{\!\!\!\!i'_{k+p}}\: C_{i'_1 ..., i'_{k+p}}, 
\vspace{-5pt}
\end{multline*}
therefore, \m{C} is a $k\<+p$'th order $P$--tensor.

\par{\textbf{Case 2.}~}
Let $C=A\<\odot_{(\sseq{a}{p})} B$. Under \rf{eq: act1} and \rf{eq: act2}, \m{C} transforms as  
\begin{multline*}
[C']_{\sseq{i}{k}} = 
\br{[P_\pi]_{i_1}{}^{\!\!i'_1} \:  \ldots\: [P_\pi]_{i_k}{}^{\!\!i'_k} \:[A]_{\sseq{i'}{k}}} 
\br{[P_\pi]_{i_{a_1}}{}^{\!\!\!i'_{a_1}} \:  \ldots\: [P_\pi]_{i_{a_p}}{}^{\!\!\!i'_{a_p}} \:
[B]_{i'_{a_1}, ... ,i'_{a_p}}}= \\
= [P_\pi]_{i_1}{}^{\!\!i'_1} \: \ldots \: [P_\pi]_{i_{k}}{}^{\!\!i'_{k}} 
\:\cdot\: [P_\pi]_{i_{a_1}}{}^{\!\!\!i'_{a_1}}\:\ldots\: [P_\pi]_{i_{a_p}}{}^{\!\!\!i'_{a_p}} \:\cdot\:  
[C]_{i'_1 ..., i'_{k}}.  
\end{multline*}
Note that each of the \m{\smash{{[P_\pi]_{i_{a_j}}{}^{\!\!\!\!i'_{a_j}}}}} factors in this expression 
repeats one of the earlier appearing 
\m{[P_\pi]_{i_1}{}^{\!\!i'_1}, \:\ldots\:, [P_\pi]_{i_k}{}^{\!\!i'_{k}}} factors, but 
since \m{P_\pi} only has zero and one entries $[P_\pi]^2_{a, b} = [P_\pi]_{a, b}$, so these 
factors can be dropped. 
Thus, $C$ is a $k$'th order $P$--tensor.

\par{\textbf{Case 3.}~}
Let \m{C\<=A\tdown_{\sseq{a}{p}}} and 
\m{\sseq{b}{k-p}} be the indices (in increasing order) that are \textbf{not} amongst \m{\cbrN{\sseq{a}{p}}}.  
Under \rf{eq: act1}, \m{C} becomes 
\begin{align*}
[C']_{i_{b_1},\ldots,i_{b_{k-p}}} 
&=\sum_{i_{a_1}}\ldots \sum_{i_{a_p}} 
[P_\pi]_{i_1}{\vphantom{I}}^{\!\!i'_1} \:  \ldots\: [P_\pi]_{i_k}{\vphantom{I}}^{\!\!i'_k} \:[A]_{\sseq{i'}{k}}\\  
& = [P_\pi]_{i_{b_1}}{}^{\!\!i'_{b_1}} \:  \ldots\: [P_\pi]_{i_{b_{k-p}}}{}^{\!\!\!i'_{b_{k-p}}} 
\sum_{i'_{a_1}} \ldots \sum\limits_{i'_{a_p}}^{}  \:[A]_{\sseq{i'}{k}} 
\end{align*}
Thus, $C$ is a $k-p$'th order $P$--tensor.

\par{\textbf{Case 4.}~} Follows directly from 3. 

\par{\textbf{Case 5.}~} 
Finally, if $A_1, ..., A_u$ are $k$'th order $P$--tensors and 
$C = \sum_j \alpha_j A_j$ then 
\begin{align*}
[C']_{i_1, ..., i_k} 
& = \sum_j \alpha_j\:  [P_\pi]_{i_1}{\vphantom{I}}^{\!\!i'_1} \:  \ldots\: [P_\pi]_{i_k}{\vphantom{I}}^{\!\!i'_k}
\;[A'_j]_{i'_1, ..., i'_k} 
= [P_\pi]_{i_1}{\vphantom{I}}^{\!\!i'_1} \:  \ldots\: [P_\pi]_{i_k}{\vphantom{I}}^{\!\!i'_k}
\sum_j \alpha_k\; [A'_j]_{i'_1, \ldots, i'_k}, 
\end{align*}
so \m{C} is  a $k$'th order $P$--tensor.
\end{pfprop}

\begin{pfprop}{prop: promotion}
Under the action of a permutation \m{\pi\tin\Sbb_{m'}} on \m{\Pcal_b}, \m{\chi}
(dropping the \m{{}^{a\to b}} superscipt) transforms to \m{\chi'}, where \m{\chi'_{i,j}=\chi_{\pi^{-1}(i),j}}. 
However, this can also be written as 
\[\chi'_{i,j}=[P_\pi \chi]_{i,j}=
\sum_{i'}[P_{\pi}]_{i,i'} \chi_{i',j}.\vspace{-8pt}\]
Therefore, \m{\widetilde F_{\sseq{i}{k}}} transforms to 
\[\widetilde F'_{\sseq{i}{k}}=
\chi'_{i_1}{}^{\!\!j_1}\,\chi'_{i_2}{}^{\!\!j_2}\,\ldots\,
\chi'_{i_k}{}^{\!\!j_k}\,F_{\sseq{j}{k}}
=[P_\pi]_{i_1}{\vphantom{i}}^{\!\!i'_1} \:  \ldots\: [P_\pi]_{i_k}{\vphantom{i}}^{\!\!i'_k}\; 
\chi_{i'_1}{}^{\!j_1}\,\chi_{i'_2}{}^{\!j_2}\,\ldots\,
\chi_{i'_k}{}^{\!j_k}\:F_{\sseq{j}{k}},
\]
so \m{\widetilde F} is a \m{P}--tensor. 
\end{pfprop}

\begin{pfprop}{prop: stacking}
By Proposition \ref{prop: promotion}, under the action of any permutation \m{\pi}, each of the 
\m{\smash{\widetilde F_{p_j}}} slices of \m{\wbar F} transforms as 
\[[\widetilde F'_{p_j}]_{\sseq{i}{k}}=
[P_\pi]_{i_1}{\vphantom{i}}^{\!\!i'_1} \:  \ldots\: [P_\pi]_{i_k}{\vphantom{i}}^{\!\!i'_k}\;
[\widetilde F'_{p_j}]_{\sseq{i}{k}}.\]
At the same time, \m{\pi} also permutes the slices amongst each other according to 
\[\wbar F'_{\sseq{i}{k},j}=[\widetilde F_{p_{\pi^{-1}(j)}}]_{\sseq{i}{k}} =
\wbar F'_{\sseq{i}{k},\pi^{-1}(j)}. \]
Therefore 
\[\wbar F'_{\sseq{i}{k},j}=
[P_\pi]_{i_1}{\vphantom{i}}^{\!\!i'_1} \:  \ldots\: [P_\pi]_{i_k}{\vphantom{i}}^{\!\!i'_k}\;
[P_\pi]_{j}{\vphantom{i}}^{\!\!j'} 
\wbar F_{\sseq{i}{k},j}, 
\] 
so \m{\wbar F} is a \m{k\<+1}'th order \m{P}--tensor. 
\end{pfprop}

\begin{pfprop}{prop: adjacency}
Under any permutation \m{\pi\tin\Sbb_m} of \m{\prt_i},\: 
\m{A\tdown_{\prt'_i}} transforms to \m{A\tdown_{\prt'_i}}, 
where \m{[A\tdown_{\prt'_i}]_{\pi(a),\pi(b)}=[A\tdown_{\prt_i}]_{a,b}}. 
Therefore, \m{A\tdown_{\prt_i}} is a second order \m{P}--tensor. 
By the first case of Proposition \ref{prop: contractions}, 
$F\otimes A\tdown_{\prt_i}$ is then a \m{k\<+2}'th order \m{P}--tensor. 
\end{pfprop}

\ignore{
We are given node $n_a$ is a descendant of node $n_b$ and has receptive field
$\Pcal_a = (e_{p_1}, ..., e_{p_m})$ and $\Pcal_b = (e_{q_1}, ..., e_{q_m'})$
$F$ is a $k$th order $P$--tensor with respect to the permutations of
$\Pcal_a = (e_{p_1}, ..., e_{p_m})$. Let $\pi \in \Sbb_{|\Pcal_b|}$ and $\tilde{F}'$ be
the transformation of $\tilde{F}$ under $\pi$,
First we note that $[\tilde{F}]_{i_1, ..., i_k}$ can be expressed as a single specific index of tensor $F$:
\begin{align}
[\tilde{F}]_{i_1, ..., i_k} &= \sum_{j_1} \ldots \sum_{j_k} \chi_{i_1, j_1} \ldots
\chi_{i_k, j_k} [F]_{j_1, ..., j_k} \nonumber \\
[\tilde{F}]_{i_1, ..., i_k} &= [F]_{j'_1, ..., j'_k}, \text{ where } q_{i_t} = p_{j'_t} \text{ for } t \in \{1, ...,  k\} \label{eq: reduce}
\end{align}
We get the second equality by noting that the entries of $\chi$ are all $0$ or $1$.
Terms in the summation of the first equality are only ever non-zero when
each of $\chi_{i_1, j_1}$, ..., $\chi_{i_k, j_k}$ are $1$, which happens if and only
if $q_{i_1} = p_{j'_1}, q_{i_2} = p_{j'_2} ..., $ and $q_{i_k} = p_{j'_k}$.
After permuting the nodes in $\Pcal_b$, we have $q_{\pi(i_1)} = p_{j'_1}$, ...,
$q_{\pi(i_k)} = p_{j'_k}$.
Consider $F'$:
\begin{align}
[\tilde{F'}]_{i_1, ..., i_k} &= \sum_{j_1} \ldots \sum_{j_k}[P_\pi \chi]_{i_1, j_1} \ldots [P_\pi
\chi]_{i_k, j_k} [F]_{j_1, ..., j_k} \label{eq: first}\\
&= \sum_{j_1} \ldots \sum_{j_k} [\chi]_{\pi^{-1}(i_1), j_1} \ldots [\chi]_{\pi^{-1}(i_k), j_k}
[F]_{j_1, ..., j_k} \label{eq: sec}\\
&= [F]_{j'_1, ..., j'_k}, \text{ where } q_{\pi^{-1}(i_t)} = p_{j'_t} \text{ for }
t\in\{1, ..., k\} \label{eq: preinvperm} \\
&= [\tilde{F}]_{\pi^{-1}(i_1), ..., \pi^{-1}(i_k)} \label{eq: invperm} \\
&= \sum_{j_1} \ldots \sum_{j_k} [P_\pi]_{i_1, j_1} \ldots [P_\pi]_{i_k, j_k}[\tilde{F}]_{j_1, ..., j_k} \label{eq: prop5result}
\end{align}
Multiplying $\chi$ from the left with $P_\pi$ permutes the rows of $\chi$ according
to permutation $\pi$.
Equation (\ref{eq: invperm}) follows from (\ref{eq: preinvperm}) by applying (\ref{eq: reduce}). The summation in (\ref{eq: prop5result}) reduces to $[\tilde{F}]_{j_1, ..., j_k}$
where $\pi(j_t) = i_t$ for $t = 1, 2, ..., k$, which is true if and only if $j_t = \pi^{-1}(i_t)$
as in (\ref{eq: invperm}).
Thus, $\tilde{F}$ is a $k$'th order $P$-tensor as desired.
}
\ignore{
Proof by induction. Base case: let \m{n_a} be a leaf node in \m{\Gcal}, and \m{n'_b} be the
corresponding node in \m{\Gcal'}. \m{f_a = \labl_a},
\m{f'_b = \labl'_{\sigma^{-1}(b)} = \labl_{\sigma^{-1}(\sigma(a))} = \labl_a = f_a}.

Inductive hypothesis: assume that any node \m{n_a} that is a distance less than \m{r} from any leaf node, and has
feature vector \m{f_a = f'_b}, where node \m{n'_b} in \m{\Gcal'} is the corresponding
node of \m{n_a}.

Inductive step: for any node \m{n_a} that is a distance at most \m{r} from any leaf node in
\m{\Gcal}, let its children be: \m{n_{a_1}, n_{a_2}, ..., n_{a_k}},
and let the corresponding node \m{n'_b} have children
\m{n_{b_1}, n_{b_2}, ..., n_{b_k}}. Then for \m{i = 1,2, ..., k},
\m{f_{a_i} = f'_{b_i}} by the induction hypothesis. And so
\m{f_i = \Phi(f_{a_1}, ...f_{a_k}) = \Phi(f_{b_1}, ..., f_{a_k}) = f'_b} since the aggregation
function \m{\Phi} is permutation invariant.

Let \m{n_r} be the root node of \m{\Gcal}, and \m{n'_{\hat{r}}} be the corresponding root node
of \m{\Gcal'}.\m{\Phi(\Gcal) = f_r = f'_{\hat{r}} = \Phi(\Gcal')}. So the overall representation
of \m{\Gcal} is invariant to the permutation of its atoms as desired.
}
\ignore{ 
Let $\bar{F'}$ and $\tilde{F'}$ denote the transformation of
$\bar{F}$ and $\tilde{F}$ under permutation $\pi \in \Sbb_{\Pcal_t}$.
\begin{align*}
\bar{F'}_{i_1, ..., i_k, j} &= [\tilde{F'}_{p_j}]_{i_1, ...., i_k} \\
&= [\tilde{F}_{p_{\pi^{-1}(j)}}]_{\pi^{-1}(i_1), ..., \pi^{-1}(i_k)} \\
&= \sum_{j_1} \ldots \sum_{j_k} \sum_l [P_\pi]_{i_1, j_1} \ldots [P_\pi]_{i_k, j_k}
[P_\pi]_{j, l}
[\tilde{F}_l]_{j_1, ..., j_k}  \\
&= \sum_{j_1} \ldots \sum_{j_k} \sum_l [P_\pi]_{i_1, j_1} \ldots [P_\pi]_{i_k, j_k} [P_\pi]_{j, l}
[\bar{F}]_{j_1, ..., j_k, l}  \\
\end{align*}
Thus $\bar{F}$ is a $k+1$'th order $P$--tensor.
}